\documentclass{article}

\usepackage{PRIMEarxiv}

\usepackage{natbib}

\usepackage{graphicx} 
\usepackage{subfig}
\usepackage{wrapfig}
\usepackage{xfrac}
\usepackage{subcaption}
\usepackage{tikz}
\usepackage[utf8]{inputenc} % allow utf-8 input
\usepackage[T1]{fontenc}    % use 8-bit T1 fonts
\usepackage{hyperref}       % hyperlinks
\usepackage{url}            % simple URL typesetting
\usepackage{booktabs}       % professional-quality tables
\usepackage{amsfonts}       % blackboard math symbols
\usepackage{nicefrac}       % compact symbols for 1/2, etc.
\usepackage{microtype}      % microtypography
\usepackage{xcolor}         % colors
\usepackage{bbm}
\usepackage{amsmath}
\usepackage{amsthm}
\usepackage{caption}
\usepackage{subcaption}
\usepackage{changes} % to track changes
\usepackage{amssymb}
\usepackage{mathtools}

\usepackage{hyperref,cleveref}       % hyperlinks

\newtheorem{theorem}{Theorem}[section]
\newtheorem{lemma}{Lemma}[section]
\newtheorem{corollary}{Corollary}[theorem]
\theoremstyle{definition}

\theoremstyle{remark}

\newcommand{\KL}{\textsf{KL}}
\newcommand{\MDL}{\textnormal{MDL}}
\newcommand{\comp}{\mathsf{c}}
\newcommand{\abs}[1]{\left\lvert{#1}\right\rvert}
\newcommand{\removed}[1]{}

\newtheorem*{restatethm}{\theoremname} 

\newcommand{\restate}[2]{%
    \def\theoremname{\textbf{\autoref{#2}}} % Makes "Lemma X.Y" or "Theorem X.Y" bold
}

\title{Quantifying Overfitting along the Regularization Path for Two-Part-Code MDL in Supervised Classification
}

\author{
  Xiaohan Zhu\\
  The University of Chicago \\
  \texttt{xiaohanz@uchicago.edu} \\
   \And
  Nathan Srebro \\
  Toyota Technological Institute at Chicago\\
  \texttt{nati@ttic.edu} \\
}

\begin{document}
\maketitle

\begin{abstract}
We provide a complete characterization of the entire regularization curve of a modified two-part-code Minimum Description Length (MDL) learning rule for binary classification, based on an arbitrary prior or description language.  \citet{GL} previously established the lack of asymptotic consistency, from an agnostic PAC (frequentist worst case) perspective, of the MDL rule with a penalty parameter of $\lambda=1$, suggesting that it underegularizes. Driven by interest in understanding how benign or catastrophic under-regularization and overfitting might be, we obtain a precise quantitative description of the worst case limiting error as a function of the regularization parameter $\lambda$ and noise level (or approximation error), significantly tightening the analysis of \citeauthor{GL} for $\lambda=1$ and extending it to all other choices of $\lambda$.
\end{abstract}

% keywords can be removed

\section{Introduction}
In this paper, we consider the modified two-part-code Minimum Description Length (MDL) learning rule in supervised binary classification, given by:
\begin{align}
    \MDL_{\lambda}(S) &= \underset{h:\mathcal{X}\rightarrow\{0,1\}}{\text{arg min }} \lambda(-\log \pi(h))+ \log {m \choose mL_S(h)} \notag\\
    &\approx \underset{h:\mathcal{X}\rightarrow\{0,1\}}{\text{arg min }} \lambda(-\log \pi(h)) + mH(L_S(h)),\label{eq:mdl-intro}
\end{align}
where $L_S(h)$ is the (zero-one) training error on a labeled training set $S$ of size $m$, $H(\cdot)$ is the binary entropy and $\pi$ is a chosen prior over predictor $h$ (see \autoref{formal setup} for a complete description).

The case $\lambda=1$\removed{, as suggested by \cite{GL},} can be thought of as the length of a two-part-code description for the labels in the sample, where the encoding is specified by the prior $\pi$.  The first term corresponds to the length of the encoding of the chosen predictor $h$ using an optimal coding for source $\pi$. The second term corresponds to encoding the labels by indicating how they differ from $h$.  This view is also related to viewing $\MDL_1$ as a Maximum A-Posterior predictor, selecting the predictor maximizing the posterior $\Pr\left(h\middle| S\right)$, where $h\sim\pi$ and the labels in the sample are then generated by flipping the output of $h$ with noise probability\footnote{To make this view more precise, we need to instead draw the noise probability at random from a uniform prior. See, e.g., \citet{GL}.} $L_S(h)$.

The MDL rule can also be seen as a form of {\em regularized empirical risk minimization}, where the second term minimizes the empirical risk $L_S(h)$, and this is balanced by the first term which controls {\em complexity}, where very low prior $\pi(h)$ corresponds to high complexity, and the form of complexity control is specified by $\pi$.  

However, as noted by \cite{GL}, this penalization is suboptimal in an agnostic setting, where we would like to compete with some $h^*$ with low $(-\log \pi(h^*))$ without the model assumption $Y|X=Y|h^*(X)$ (See \autoref{wellspecifiedsection} for a discussion of the well-specified case, where this modeling assumption {\em is} made).  \citeauthor{GL}  showed that in this case, we might not have consistency (i.e.~strong learning), in the sense that even as the number of samples $m$ increases, the limiting population error might not be optimal: $\lim_{m\rightarrow \infty} \mathbb{E}[L(\MDL_1(S))] > L(h^*)$, where $L(h)$ is the population (zero-one) error.

Instead, if we would like to compete with an unknown predictor $h^*$ with low complexity $-\!\log\pi(h^*)$ (i.e.~prior $\pi(h^*)$ away from zero), the Structural Risk Minimization (SRM) principal (\citet{vapnik1991principles}, discussed specifically in our setting by \citet{GL}, and see also \citet{SRM} Section 7.3)
suggests a different balance of empirical risk and prior:
\begin{equation}\label{eq:srm}
    \textnormal{SRM}(S) = \arg\min_h L_S(h)+\sqrt{\frac{-\log \pi(h)}{m}} = \arg\min_h \sqrt{m}\sqrt{-\log \pi(h)} + m L_S(h)
\end{equation}
This balance {\em does} ensure consistency, and even with a finite sample guarantee, where with high probability  $L(\textnormal{SRM}(S)) \leq L(h^*) + O\left(\sqrt\frac{-\log \pi (h^*)}{m}\right)$ and so $\limsup_{m\rightarrow\infty}\mathbb{E}\left[L(\textnormal{SRM}(S))\right]\leq L(h^*)$.  The balance between empirical risk and regularization in \eqref{eq:srm} roughly corresponds to a choice of $\lambda_m \propto \sqrt{m}$ in \eqref{eq:mdl-intro}.  Indeed, in \autoref{lambda infty} and \autoref{lambda_infty_coro}, we obtain similar consistency guarantee using  $\MDL_{\lambda_m}$ with $\lambda_m=\sqrt{m}$.
% \natinote{or change this to the optimal choice?}.

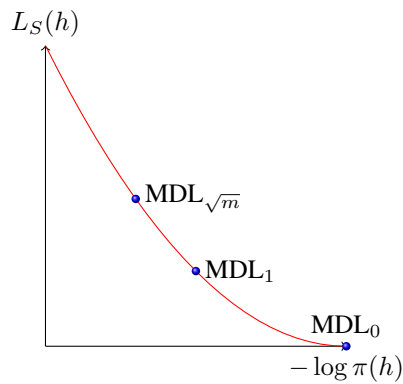
\begin{wrapfigure}{r}{0.5\textwidth}
\centering
        \begin{tikzpicture}[scale=0.8,domain=0.001:5.0]
          \draw[->] (0.001,0) -- (5.0,0) node[below] {$-\log \pi(h)$};
          \draw[->] (0,0.001) -- (0,5.0) node[above] {$L_S(h)$};
       
          \draw[color=red]    plot (\x,{0.2 * (\x-5)^2});
          \draw plot[mark=ball] coordinates{(1.5,2.45)} node[right] {$\MDL_{\sqrt{m}}$};
          \draw plot[mark=ball] coordinates{(2.5,1.25)} node[right] {$\MDL_1$};
          \draw plot[mark=ball] coordinates{(5.0,0.0)} node[above] {$\MDL_0$};
        \end{tikzpicture}
\caption{Pareto Frontier.}
\label{fig:}
\end{wrapfigure}
\vspace{5pt}
We can understand this in terms of the regularization path, depicted in \autoref{fig:}, formed by considering $\MDL_\lambda$ for different tradeoff parameters $\lambda$, as the Pareto-frontier for minimizing the empirical error $L_S(h)$ (the vertical axis) on one hand and $-\log \pi(h)$ (the horizontal axis) on the other hand.  The ``correct'' amount of regularization is $\lambda\approx\sqrt{m}$, while the choice $\lambda=1$ under-regularizes, and thus over-fits, and results in suboptimal population error.  As we decrease $\lambda$ we regularize even less and overfit more, and as $\lambda\rightarrow 0$ we approach the max prior interpolating solution\footnote{We slightly overload notation by denoting this max prior interpolator $\MDL_0$ even though it is only approached as $\lambda\rightarrow 0$.  We also implicitly assume interpolation is possible, otherwise this is the max prior predictor among all those with minimal risk.} $\MDL_0 = \arg\min_{L_S(h)=0} (-\log\pi(h))$.  
The inconsistency result of \citeauthor{GL} can thus be seen as a lower bound on the limiting ``cost of overfitting'' (i.e.~deterioration in limiting error), for a particular amount of under-regularization (see also \autoref{temperfct section}). But their analysis does not provide an upper bound on the cost of overfitting, even for the particular choice $\lambda=1$.  How bad is this overfitting?  We know it is not benign, in the sense that the limiting error is greater than $L(h^*)$, but how bad is it?  Is it catastrophic in the sense that the limiting error can be arbitrarily high?  Or is it tempered \citep{mallinar} in the sense that it can be bounded in terms of $L(h^*)$ ?

And what happens for other choices of $\lambda$, possibly $\lambda_m$ scaling with $m$?  Can we characterize the limiting error and cost of overfitting along the entire regularization path, as a function of $\lambda$ and $L(h^*)$?  Is there still a cost to overfitting when $\lambda>1$?  Up to what point?  And how bad, tempered or catastrophic, can overfitting be when we under-regularize even more, with $\lambda<1$? And what is the dependence on $L(h^*)$?  If overfitting is tempered, the dependence on $L(h^*)$ can be thought of as a `tempering function', telling us how the limiting error is bounded in terms of the noise level $L(h^*)$.

We provide (nearly) complete answers to the above questions. In particular:
\begin{itemize}
    \item We obtain a tight characterization of the worst possible limiting error as $m\rightarrow \infty$, for any $0<\lambda<\infty$ and any value of $L^*=L(h^*)$ (\autoref{cor:finite} and equation \eqref{ell}). We show that for any $1\leq\lambda<\infty$ and any $L^*>0$, we have tempered overfitting, with a precisely characterizable ``tempering functions'' (depicted in \autoref{fig2}). For $0<\lambda<1$, overfitting is tempered only for small enough noise $L^*$, again in a way we precisely characterize (also in \autoref{fig2}). Overfitting gets worse as  $\lambda$ decreases (\autoref{fig3}) and for any $\lambda_m \rightarrow 0$ we can get catastrophic overfitting for any $L^*>0$ (\autoref{lambda0}).  
    \item On the other hand, any $1 \ll \lambda_m \ll \frac{\log m}{m}$ leads to consistency, namely, we have $\underset{m\rightarrow \infty}{\limsup } \mathbb{E}[L(\MDL_{\lambda_m}(S))] \leq L(h^*) $ as long as $\pi(h^*)>0$ (\autoref{lambda_infty_coro}).  But once $\lambda_m = \Omega(m)$, we are over-regularizing and might ``underfit'', again resulting in a catastrophic behaviour where we can have  $\lim_{m\rightarrow \infty} \mathbb{E}[L(\MDL_{\lambda_m}(S))] = \frac{1}{2}$ (\autoref{lambda>>m}).
    \item In all regimes, we provide concrete finite sample upper bounds on how the error approaches the limiting error as $m\rightarrow\infty$ (\autoref{MDL UB} and \autoref{lambda infty}).
    \item For the special case $\lambda=1$ our results tighten the analysis of \citeauthor{GL} from both above and below, showing a higher lower bound than they obtained, as well as a matching upper bound (see \autoref{fig4} and discussion at the end of \autoref{temperfct section}).
\end{itemize}

Our analysis is agnostic and worst-case, both over the source distribution and over the choice of (discrete) prior $\pi$, which we can think of as a complexity measure.  In the past years there has been much study of overfitting with respect to a variety of different complexity measures, including the Euclidean norm for linear predictors \citep[e.g.][]{hastie2022surprises,Bartlett_2020}, RKHS norms \citep[e.g.][]{2019arXiv191101544M,theo,Mei2019TheGE}, other norms such as the $\ell_1$ norm \cite[e.g.][]{Ju,Wang2021TightBF,Koehler2021UniformCO}, norms of weights in neural networks \citep[e.g.][]{KYS,2023freibenign,nirmit}, program length \citep[e.g.][]{MS}, and neural network size \citep[e.g.][]{harel}. These can all be seen as studying the effect of ovefitting for {\em specific} priors $\pi$, sometimes only for the max prior interpolator $\MDL_0$, sometimes also for the entire path \cite[e.g.][]{entireregpathforridge}. Although many are continuous priors, some are discrete (e.g.~program length and neural network size).  Our work can be seen as providing a {\em baseline} highlighting the limits of the cost of overfitting for {\em any} (at least discrete) prior.  This  frames the study of overfitting along the regularization path of specific priors in terms of how they potentially improve over this baseline and reduce the cost of overfitting.  Although our framework and result captures only discrete priors, we believe the behaviour for continuous priors is similar and can also be studied.

\section{Formal Setup}
\label{formal setup}
We consider a supervised binary classification problem where we observe $m$ i.i.d samples $S \sim \mathcal{D}^m$ from a source distribution $\mathcal{D}$ over $(X,Y) \in \mathcal{X}\times \{0,1\}$, where $\mathcal{X}$ is some measurable space and $Y$ is a binary label. A ``predictor'' (aka classifier) is a (measurable) mapping $h:\mathcal{X}\rightarrow\{0,1\}$ and we are interested in its population error $L(h)=L_D(h)= \mathbb{P}_{(x,y) \sim \mathcal{D}}(h(x) \neq y)$ (we frequently omit the distribution $D$ when it is clear from context).  We also denote $L_S(h) = \frac{1}{m} \sum_{i = 1}^m \mathbbm{1}\{h(x_i) \neq y_i\}$ the empirical error (i.e.~training error) on $S$. 
\removed{A learning rule is a mapping $\mathcal{A}: S \rightarrow \{0,1\}^\mathcal{X}$ that takes training sample as input and gives a predictor as output. }

We consider learning rules based on a given ``prior''  $\pi: \{0,1\}^\mathcal{X} \rightarrow [0,1]$ over predictors such that $\sum_h \pi(h) \leq 1$.  We can think of this prior as a discrete distribution over predictors (if $\sum_h \pi(h) \leq 1$ the remaining probability mass can be thought to be absorbed in some alternate predictor), but we never sample from it or assume anything is sampled from it.  We can equivalently\footnote{Every prefix-unambiguous description language specifies a valid prior with $\pi(h) = 2^{-\abs{h}_{\pi}}$ (by Kraft's inequality), while for every valid prior there is a description language with $-\log \pi(h) \leq \abs{h}_{\pi} \leq -\log \pi(h)+1$ \citep[e.g.][Section 5.2--5.3]{TC}. } view the prior $\pi(h)$ as corresponding to a description length $\abs{h}_{\pi}$ of predictors in some prefix-unambiguous description language with $\pi(h) = 2^{-\abs{h}_{\pi}}$.  Maximizing $\pi$ is thus the same as minimizing the description length $\abs{h}_{\pi}$.  Either way, $\pi$ will have finite or countable support (i.e.~the description language can describe finite or countably many predictors), and for any countable class of predictors we can construct a prior assigning positive probability to all predictors in the class. Formally, we denote $\abs{h}_{\pi} = -\log \pi(h)$.

For a given prior $\pi$ (or equivalently, description language) and regularization parameter $\lambda$ we consider the modified Minimum Description Length learning rule:
\begin{align}\label{defn:MDL}
\MDL_{\lambda}(S) = \underset{h \in \mathcal{H}}{\text{arg min }} J_{\lambda}(h,S), \text{ such that } L_S(h)\leq 1/2,
\end{align}
where 
\begin{align}\label{stirling}
    J_{\lambda}(h,S) &= \lambda \abs{h}_{\pi} + \log {m \choose mL_S(h)} = \tilde{J}_{\lambda}(h,S) -\tilde{\Delta}, \text{ with }\notag\\
    \tilde{J}_{\lambda}(h,S) &=    \lambda \abs{h}_{\pi} + mH(L_S(h)) \text{ and } 0 \leq \tilde{\Delta} \leq \log (m+1).
\end{align}
with the equality following from Stirling's approximation.  We can also use the approximate form to define the alternative and very similar rule:
\begin{align}
 \widetilde{\MDL}_{\lambda}(S) = \underset{h}{\text{arg min }} \tilde{J}_{\lambda}(h,S) \text{ such that } L_S(h)\leq 1/2.\notag
\end{align}
All the results in the paper, including both upper and lower bounds, also apply to the $\widetilde{\MDL}_{\lambda}$ learning rule, which has the same limiting behaviour as $\MDL_{\lambda}$.

The standard MDL is then a special case of $\MDL_{\lambda}$ with $\lambda = 1$. We denote it as $\MDL_1$. Notice that when we define $\MDL_{\lambda}$, we require $ L_S(h)\leq 1/2$. This is because otherwise we are not preferring a predictor with very low error $L$ over a predictor with very high error $1-L$, e.g.~differentiating between a predictor $h(x)$ and its negation $1-h(x)$, and cannot possibly ensure $\MDL$ returns a predictor with small (rather than large) error.  If the prior is symmetric, i.e.~$\pi(h)=\pi(1-h)$, we can think of the constraint as specifying we output $1-h$ if $L_S(h)>1/2$.

\paragraph{Notation}

$\text{Ber}(\alpha)$ denotes a Bernoulli random variable with expectation $\alpha$. For a random variable $X$, $H(X)$ is its entropy, and for $\alpha,\beta\in[0,1]$ we also use $H(\alpha)=-\alpha \log \alpha-(1-\alpha)\log (1-\alpha)$ and $\KL(\alpha \Vert \beta)=\alpha\log\frac{\alpha}{\beta} + (1-\alpha)\log\frac{1-\alpha}{1-\beta}$ to denote the entropy and KL-divergence of corresponding Bernoullis. All logarithms are base-$2$ and entropy is measured in bits. We use  $a \oplus b$ to denote the XOR of two bits $a,b\in\{0,1\}$. 

\section{Main Results}
With these definitions, we are ready to state our main result:
For any $0<\lambda<\infty$, we show that the worst-case limiting error is given by the following function $\ell_{\lambda}$ plotted in \autoref{fig2}:
\begin{gather}\label{ell}
\ell_{\lambda}(L^*)=\removed{ Q_{\lambda}^{-1}(H(L^*))=}
\begin{cases}
1 - 2^{-\frac{1}{\lambda}H(L^*)},  & \text{for } 0< \lambda \leq 1 \\
U_{\lambda}^{-1}(H(L^*)), & \text{for } \lambda > 1,
\end{cases} 
\quad \textrm{where: } { U_{\lambda}(q) = \lambda \KL\!\left(\tfrac{1}{1+ \left(\frac{\scriptscriptstyle 1-q}{\scriptscriptstyle q}\right)^{\frac{\scriptscriptstyle \lambda}{\scriptscriptstyle \lambda - 1}}}\middle\Vert q\right) \!+\! H\!\left(\tfrac{1}{1+ \left(\frac{\scriptscriptstyle 1-q}{\scriptscriptstyle q}\right)^{\frac{\scriptscriptstyle \lambda}{\scriptscriptstyle \lambda - 1}}}\right)}.
\end{gather}
\removed{where $Q_{\lambda}(q) = \min_{0 \leq p \leq 0.5}{\lambda}\KL(p \Vert q) + H(p)$, and $U_{\lambda}(q) = \lambda \KL(\frac{1}{1+ (\frac{1-q}{q})^{\frac{\lambda}{\lambda - 1}}}\Vert q) + H(\frac{1}{1+ (\frac{1-q}{q})^{\frac{\lambda}{\lambda - 1}}})$.}

\begin{theorem}[Agnostic Upper Bound]
\label{MDL UB}
(1) For any $0<\lambda\leq 1$,  any source distribution $D$, any predictor $h^*$, any valid prior $\pi$, and any $m$:
\begin{align}
    \underset{S \sim D^m}{\mathbb{E}}[L(\MDL_{\lambda}(S))]
    \leq 1 - 2^{-\frac{1}{\lambda}H(L(h^*))} + O\left(\frac{\abs{h^*}_{\pi}}{m} + \frac{1}{\lambda}\sqrt{\frac{\log^3 (m)}{m}}\right).
\end{align}

(2)
For any $\lambda > 1$,  any source distribution $D$, any predictor $h^*$, any valid prior $\pi$, and any $m$:
\begin{align}
% \begin{split}
    \underset{S \sim D^m}{\mathbb{E}}[L(\MDL_{\lambda}(S))]
    \leq U_{\lambda}^{-1}(H(L(h^*)))
    + O\left(\frac{1}{(1 - 2L(h^*))^2} \cdot \left(\lambda\left(\frac{\abs{h^*}_{\pi} + \log m}{m}\right) + \sqrt{\frac{\log^3 (m)}{m}}\right)\right).
% \end{split}
\end{align}
Where $O(\cdot)$ only hides an absolute constant, that does not depend on $D, \pi$ or anything else.
\end{theorem}

To establish the exact worst-case limiting error, we provide matching lower bounds, showing that the limiting error can approach $\ell_\lambda(L^*)$, for any $0<\lambda<\infty$ and $L^*$:

\begin{theorem}[Agnostic Lower Bound]
\label{MDL LB}
    For any $0<\lambda<\infty$, any $L^* \in (0,0.5)$ and $L^* \leq L' <\ell_{\lambda}(L^*)$, there exists a prior $\pi$, a hypothesis $h^*$ with $\pi(h^*) \geq 0.1$ and source distribution $D$ with $L_D(h^*) = L^*$ such that $\mathbb{E}_S \left[L_D(\MDL_{\lambda}(S))\right] \rightarrow L'$ as sample size $m \rightarrow \infty$. 
\end{theorem}

Combining \autoref{MDL UB} and \autoref{MDL LB}, we see that $\ell_\lambda(L^*)$ given in \eqref{ell} exactly and tightly characterizes the worst case limiting error: for any $0<\lambda<\infty$, and any $L^* \in (0,0.5)$, 
\[
\underset{\substack{\pi , D,
L(h^*) = L^*\\ \pi(h^*) \geq 0.1}}{\sup} \underset{m \rightarrow \infty}{\lim} \underset{S \sim D^m}{\mathbb{E}}\left[ L_D(\MDL_{\lambda} (S))\right] = \ell_{\lambda}(L^*). 
\]

Furthermore, this convergence is ``uniform'', in the sense that we have a finite-sample guarantee (see \autoref{MDL UB}) with sample complexity (i.e rate of convergence) that depends only on $\pi(h^*)$, $\lambda$ and\footnote{The dependence on $L^*$ only kicks in when $L^*$ is close to $1/2$.  As long as $L^*$ is bounded away from $1/2$, we can ignore this dependence.} $L^*$ but not on $\pi$ and $D$.  Another way to view this is that we get the same guarantee even if we change the order of the limits. This is our main result, and is captured by the following corollary:

\begin{corollary}\label{cor:finite}
For any $0<\lambda<\infty$, and any $L^* \in (0,0.5)$, 
\[
\ell_{\lambda}(L^*) = \underset{\pi, D }{\sup } \underset{m \rightarrow \infty}{\lim} \underset{S \sim D^m}{\mathbb{E}}\left[ L_D(\MDL_{\lambda} (S))\right] \leq \underset{m \rightarrow \infty}{\lim} \underset{\pi, D }{\sup} \underset{S \sim D^m}{\mathbb{E}}\left[ L_D(\MDL_{\lambda} (S))\right] = \ell_{\lambda}(L^*).
\]
and so the inequality is actually an equality.
\end{corollary}

The analysis above allows us to describe the overfitting behaviour of $\MDL_\lambda$ for any {\em fixed} $0<\lambda<\infty$ (i.e.~not varying with $m$).  In the next Section, we study the limiting error function $\ell_\lambda(L^*)$, and see that for fixed $0<\lambda<\infty$, overfitting is never benign, but it is tempered when $\lambda \geq 1$ or $L^*$ is small enough relative to $\lambda$.  

We now turn to characterizing the behaviour when $\lambda_m$ varies with $m$, with either $\lambda_m \rightarrow 0$ or $\lambda_m \rightarrow \infty$.  At $\lambda = 0$ or $\lambda_m \rightarrow 0$, we get catastrophic overfitting with the limiting error 1:
\begin{theorem}\label{lambda0}
    For any $\lambda_m \rightarrow 0$ or $\lambda = 0$, any $L^* \in (0,0.5)$, and $L^* \leq L' < 1$, there exists a prior $\pi$, a hypothesis $h^*$ with $\pi(h^*) \geq 0.1$ and source distribution $D$ with $L_D(h^*) = L^*$ such that $\mathbb{E}_S \left[L_D(\MDL_{\lambda_m}(S))\right] \rightarrow L'$ as sample size $m \rightarrow \infty$.
\end{theorem}

As $\lambda_m \rightarrow \infty$ with $1 \ll \lambda_m \ll \frac{m}{\log m}$, we get consistency, i.e. ``learning'' behaviour with the following finite sample guarantee:

\begin{theorem}
\label{lambda infty}  
For any predictor $h^*$, source distribution $D$, valid prior $\pi$, and any $\lambda>1$ and m:
\begin{equation}\label{thm3.4bd}
    \underset{S \sim D^m}{\mathbb{E}}[L(\MDL_{\lambda}(S))] 
    \leq L(h^*)+ O\left(\frac{1}{1 - 2L(h^*)} \cdot \left(\frac{1}{\lambda} + \lambda\left(\frac{\abs{h^*}_{\pi} + \log m}{m}\right) + \sqrt{\frac{\log^3 (m)}{m}}\right)
    \right),
\end{equation}
where $O(\cdot)$ only hides an absolute constant, that does not depend on $D, \pi$ or anything else.
\end{theorem}

As with the finite sample guarantee of \autoref{MDL UB}, the factor $\frac{1}{1 - 2L(h^*)}$ is bounded as long as $L(h^*)$ is bounded away from $0.5$, and we can get consistency for any $L(h^*)<0.5$.

The optimal setting for $\lambda_m$ in \autoref{lambda infty} is $\lambda_m=\sqrt{\frac{m}{\abs{h^*}_\pi+\log m}}$, and with any $\lambda_m \propto \sqrt{m}$ we get consistency with rate $\propto \tilde{O}(1/\sqrt{m})$.  More broadly, to get consistency, we need $\lambda_m\rightarrow\infty$  to ensure that $\frac{1}{\lambda}$ vanishes, but also not too fast such that the term $\lambda\left(\frac{\abs{h^*}_{\pi} + \log m}{m}\right)$ also vanishes. This gives the following corollary:

\begin{corollary}\label{lambda_infty_coro}
    For $1 \ll \lambda_m \ll \frac{m}{\log m}$, and any $h^*$ with $\pi(h^*)>0$ and $L(h^*)<0.5$, we have $\underset{m \rightarrow \infty}{\lim} \underset{\pi, D }{\sup} \underset{S \sim D^m}{\mathbb{E}}[L(\MDL_{\lambda_m})] \leq L(h^*)$.
\end{corollary}

However, when $\lambda_m = \Omega(m)$, $\MDL_{\lambda_m}$ over-regularizes and leads to catastrophic behavior again:
\begin{theorem}\label{lambda>>m}
    % For any $\lambda_m = \Omega(m)$ with $\liminf \frac{\lambda_m }{m}> 10$, any $L^* \in (0,0.5)$ and $L^* \leq L' < 1$, there exists a prior $\pi$, a hypothesis $h^*$ with $\pi(h^*) \geq 0.4$ and source distribution $D$ with $L_D(h^*) = L^*$ such that $\mathbb{E}_S \left[L_D(\MDL_{\lambda_m}(S))\right] \rightarrow L'$ as sample size $m \rightarrow \infty$.
    For any $\lambda_m = \Omega(m)$ with $\liminf \frac{\lambda_m}{m}> 10$, any $L^* \in [0,0.5)$, and any $L^* \leq L' < 0.5$, there exists a prior $\pi$, a hypothesis $h^*$ with $\pi(h^*) \geq 0.1$ and source distribution $D$ with $L_D(h^*) = L^*$ such that $\mathbb{E}_S \left[L_D(\MDL_{\lambda_m}(S))\right] \rightarrow L'$ as sample size $m \rightarrow \infty$.
\end{theorem}

This gives an (almost) complete picture of the worst case limiting error of $\MDL_{\lambda_m}$, both when $\lambda_m$ is fixed\footnote{Relying on the finite-sample guarantees in \autoref{MDL UB}, it is also possible to analyze the case where $\lambda_m$ varies with $m$ but has a finite positive limit.} as well as when $\lambda_m$ increases or decreases with $m$:
\begin{description}
    \item[$\mathbf{\lambda_m\rightarrow 0}$:] In this case we have catastrophic over-fitting for any $0<L^*<1/2$, with worst case limiting error :
    \begin{equation}
 \underset{\pi, D }{\sup } \underset{m \rightarrow \infty}{\lim} \underset{S \sim D^m}{\mathbb{E}}\left[ L_D(\MDL_{\lambda_m} (S))\right] = \underset{m \rightarrow \infty}{\lim} \underset{\pi, D }{\sup} \underset{S \sim D^m}{\mathbb{E}}\left[ L_D(\MDL_{\lambda_m} (S))\right] = 1
    \end{equation}
 
    \item[$\mathbf{0<\lambda<\infty}$:] In this case the limiting error is governed by $\ell_\lambda(L^*)$, and discussed further in the next Section.
    
\item[$\mathbf{\lambda_m\rightarrow\infty}$ but $\mathbf{\lambda_m = o\left(\frac{m}{\log(m)}\right)}$]: In this case we have consistency (i.e.~strong learning) and for any $0\leq L^* < 1/2$: 
   \begin{equation}
 \underset{\pi, D }{\sup } \underset{m \rightarrow \infty}{\lim} \underset{S \sim D^m}{\mathbb{E}}\left[ L_D(\MDL_{\lambda_m} (S))\right] = \underset{m \rightarrow \infty}{\lim} \underset{\pi, D }{\sup} \underset{S \sim D^m}{\mathbb{E}}\left[ L_D(\MDL_{\lambda_m} (S))\right] = L^*
    \end{equation}
    
    \item[$\mathbf{\lambda_m = \Omega(m)}$:] In this case we are catastrophically underfitting and for any $0 \leq L^* < 1/2$:
   \begin{equation}
 \underset{\pi, D }{\sup } \underset{m \rightarrow \infty}{\lim} \underset{S \sim D^m}{\mathbb{E}}\left[ L_D(\MDL_{\lambda_m} (S))\right] = \underset{m \rightarrow \infty}{\lim} \underset{\pi, D }{\sup} \underset{S \sim D^m}{\mathbb{E}}\left[ L_D(\MDL_{\lambda_m} (S))\right] = 1/2
    \end{equation}    
 \end{description}
This is an almost complete description, with a gap between $m/\log m$ and $10m$, which is discussed further in \autoref{summary section}.  

\paragraph{Relationship  with \cite{MS}}  Our work was inspired by that of \citet{MS}, who studied (in our language) the interpolating $\MDL_0$ learning rule for a ``Kolmogorov'' prior $\pi$, where $\abs{h}_\pi$ is the minimum program length\footnote{This is almost equivalent to a prior over programs, where characters are generated uniformly at random, until a valid program, in a prefix unambiguous programming language, is reached \citep[e.g.][Appendix A]{buzaglouniform}.} for $h$.  They demonstrated that with the Kolmogorov prior, the tempering behaviour at $\lambda=0$ is given by a tempering function equal to our $\ell_1(L^*)$.  That is, the specific Kolmogorov prior behaves better than the worst case prior for $\lambda=0$ (since the worst case behavior at $\lambda=0$ is always catastrophic).  Our setting and questions are thus very different from theirs (we consider a worst case prior, while their analysis was very specific to the Kolmogorov prior, and we consider any $\lambda$ while they only considered $\lambda=0$), but our core upper bound analysis was inspired by theirs and builds on a non-realizable generalization of the same information-theoretic generalization guarantee. 

\begin{figure}[h]
    \centering
    \begin{minipage}{0.55\textwidth}
        \centering
        \includegraphics[width=\linewidth]{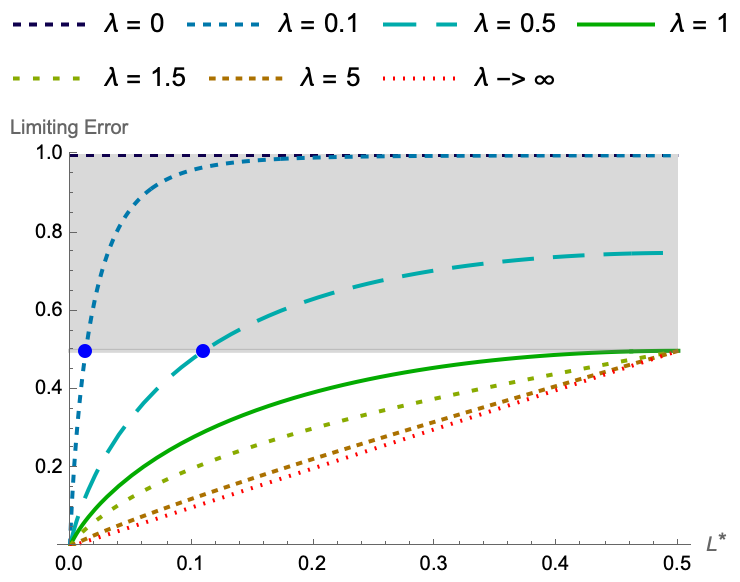}

    \end{minipage}%
    \hfill
    \begin{minipage}{0.45\textwidth}
        \small
        \caption{\small 
  Agnostic worst-case limiting error $\ell_\lambda(L^*)$ (see \autoref{cor:finite} and equation \eqref{ell}) as a  function of the noise level $L^*$, for different $\lambda$.  For each noise level $L^*=L(h^*)$, the curve indicates the best possible guarantee on the limiting error.  As $\lambda\rightarrow\infty$ the tempering curve approaches the diagonal $\ell(L^*)=L^*$, indicating consistency.  For $\lambda<\infty$, the curve is strictly above the diagonal, i.e.~$\ell(L^*)>L^*$ (for $0<L^*<0.5$), and we do not have consistency.  For $\lambda \geq 1$, the curve is always below $0.5$ (the unshaded bottom half of the figure), indicating that for any noise level $L^*<0.5$ overfitting is ``tempered'' in that the limiting error is better than chance.  But for $\lambda<1$, this is only the case for $L^*<L_\textrm{critical}=H^{-1}(\lambda)$, and this critical point is indicated by the blue dots on the curves for $\lambda=0.1,0.5$.  For $\lambda=0$ the worst case limiting error is always 1.}
   \label{fig2}
    \end{minipage}

\end{figure}

\section{The Tempering Function $\ell_{\lambda}(L^*)$}\label{temperfct section}

In the previous Section we obtained an exact characterization of the worst-case limiting error $\ell_\lambda(L^*)$ as a function of the noise level (or error assumption $L(h^*)=L^*$ on the reference predictor $h^*$ with which we are competing), and tradeoff parameter $\lambda$.  This explicit function is plotted in \autoref{fig2} for several values of $\lambda$.

We can see that for $\lambda \geq 1$, the limiting error $\ell_{\lambda}$ is a continuous 1:1 function from $[0,\frac{1}{2}]$ to $[0,\frac{1}{2}]$,  i.e. for any $L^*<\frac{1}{2}$ we have $\ell_{\lambda}(L^*)<\frac{1}{2}$. Hence, the guaranteed overfitting still gives us ``weak learning'' whenever $L^*<0.5$ (i.e.~the reference is better than chance) in the sense that $\underset{m \rightarrow \infty}{\lim}  \underset{\pi, D }{\sup}\underset{S \sim D^m} {\mathbb{E}} \left[ L_D(\MDL_{\lambda} (S))\right] < \frac{1}{2}$, which is at least better than random guessing. However, studying the behaviour about $L^*=0$, we can calculate that the derivative with respect to $L^*$ (the slope of the depicted curve) explodes as $L^*\rightarrow 0$ (i.e.~$\ell'_\lambda(L^*)\rightarrow\infty$), for any $\lambda<\infty$.  This means that although overfitting is ``tempered'' in the sense that we can ensure error better than random guessing, there is no $C_\lambda$ such that $\ell_\lambda(L^*)\leq C_\lambda L^*$, i.e.~the ratio between the limiting error and reference error is unbounded.

On the other hand, for $\lambda<1$, although $\ell_\lambda(L^*)$ is still continuous and 1:1 w.r.t.~$L^*$, and
we still have $\ell_{\lambda}(L^*) \rightarrow 0$ as $L^* \rightarrow 0$, we get tempered overfitting (limiting error better than chance) only if $L^*$ is small enough, specifically lower than some finite critical $L_\textnormal{critical} = H^{-1}(\lambda)$.  If $L^*>L_\textnormal{critical}$, $\MDL_{\lambda}$ is useless since its limiting error can be as bad, or even worse, than random guessing.  The two blue points in \autoref{fig2} indicate this critical point where $\ell_{0.1}$ and $\ell_{0.5}$ hit $\frac{1}{2}$.

As $\lambda_m \rightarrow \infty$, the cost of overfitting vanishes and the tempering function approaches the ``consistent'' $\ell_{\infty}(L^*) = L^*$.  This matches \autoref{lambda infty}, which ensures consistency once $\lambda_m \rightarrow \infty$ (but not too fast, least we start {\em over}regularizing and {\em under}fitting---this effect cannot be seen in the Figures and through $\ell_\lambda$, which only indicates {\em over}fitting behavior for finite $\lambda$).

We can also see the increasing cost of overfitting as $\lambda$ decreases in \autoref{fig3}, which depicts the limiting error $\ell_\lambda(L^*)$ as a function of $\lambda$ for a particular noise level $L^*=0.1$. As long as $\lambda>H(L^*)$, the limiting error is lower than half, we have weak learning and overfitting is ``benign''.  Though we will only have consistency as $\lambda\rightarrow\infty$ and the limiting error curve asymptotes to the noise (or reference error) level $L^*$.  But below the critical $\lambda=H(L^*)$, overfitting is catastrophic and the limiting error is not guaranteed to be better than $0.5$.

\begin{figure}[h]
    \centering
    % First figure
    \begin{minipage}{0.48\textwidth}
        \centering
        \includegraphics[width=\linewidth]{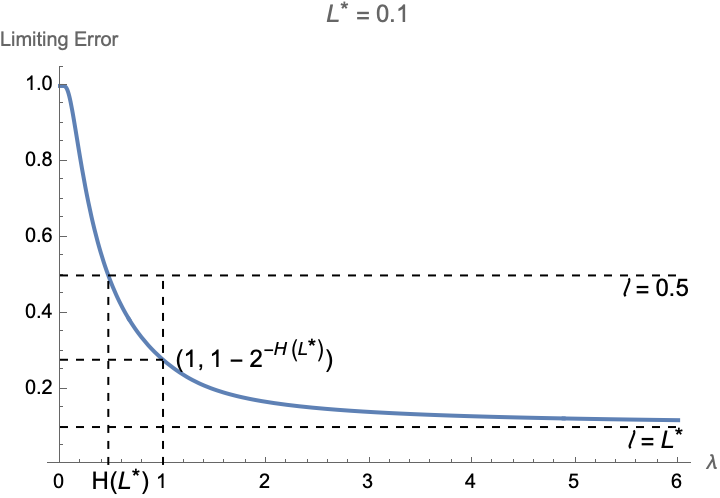}  % First figure
        \small \caption{\small Agnostic worst-case limiting error $\ell_\lambda(L^*)$ of $\MDL_{\lambda}$ as a function of $\lambda$, at a fixed noise level $L^* = 0.1$.
        The error curve is a continuous function of $\lambda$ for $0\leq\lambda<\infty$. } \label{fig3}
    \end{minipage}
    \hfill
    % Second figure
    \begin{minipage}{0.48\textwidth}
        \centering
        \includegraphics[width=\linewidth]{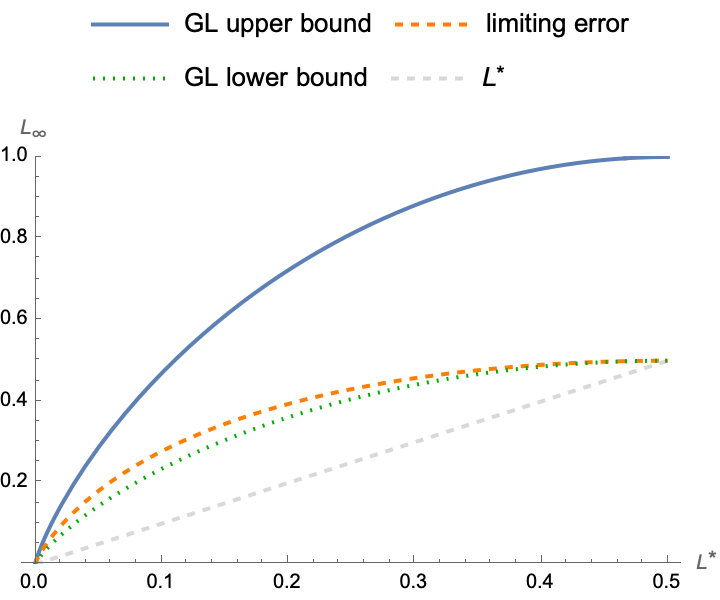}  % Second figure
        \small  \caption{\small Comparison to \citet{GL},  for the case $\lambda=1$. Their lower bound for the limiting error of $\MDL_1$ is in green. Our matching lower and upper bounds are in red. Also shown in blue is their upper bound for the related Bayes predictor (they do not  provide an upper bound for $\MDL_1$).}
  \label{fig4}
    \end{minipage}
\end{figure}

\paragraph{Comparison with \citet{GL}}
\citeauthor{GL} showed that for $\lambda=1$, the worst case limiting error of $\MDL_1$ is lower bounded by $H(L^*)/2$, which is the green line plotted in \autoref{fig4}, thus worse than $L^*$ for $0<L^*<0.5$.  In our terminology, they showed that $\ell_1(L^*)>H(L^*)/2>L^*$ for $0<L^*<0.5$.  They could not provide an upper bound, and left open how bad the limiting error for $\MDL_1$ could be. Instead they showed an upper bound of $H(L^*)$, depicted in blue in \autoref{fig4}, only for the related but stronger Bayes predictor.  Specializing to $\lambda=1$, we provide a tighter lower bound $\ell_1(L^*) = 1 - 2^{-H(L^*)}>H(L^*)/2$ (for $0<L^*<0.5)$, which is the red line in the figure.  We also provide, for the first time, an upper bound on the limiting error of $\MDL_1$ (as opposed to the Bayes predictor), thus establishing the exact worst case limiting error (the red curve in \autoref{fig4}).  Furthermore, our upper bound is backed up by a finite sample guarantee. Thus, even specializing to the case $\lambda=1$, we significantly improve on the analysis of \citeauthor{GL}.

\section{Generalization Guarantees and Proof of Upper Bounds}

In this Section, we describe our proof technique and provide proof sketches for our upper bound \autoref{MDL UB} and \autoref{lambda infty}.  Recall that these Theorems provide finite sample guarantees on the error $\MDL_\lambda$, which imply upper bounds on the limiting error (\autoref{cor:finite} and \autoref{lambda_infty_coro}). Complete proof details can be found in \autoref{AppendixA}. 

Our upper bounds are based on the following core generalization guarantee:
\begin{lemma}\label{lemma8}
For some constant $C$, any $0 <\lambda < \infty$, with probability $1- \delta$ over $S\sim\mathcal{D}^m$, for any predictor $h^*$:
\begin{gather}
    Q_{\lambda}(L(\MDL_{\lambda}(S))) \leq H(L(h^*)) + {\lambda}\frac{\log(\frac{m+1}{\delta/2})}{m}+{\lambda}\frac{\abs{h^*}_{\pi}}{m} + C\sqrt{\frac{2(\log m)^2 \cdot \log \frac{1}{\delta/2}}{m}} \notag\\
\textrm{where:}\quad\quad Q_{\lambda}(q) = \min_{0 \leq p \leq 0.5}{\lambda}\KL(p \Vert q) + H(p) \label{eq:Q}
\end{gather}
\end{lemma}

\begin{proof}
We start from a concentration guarantee, expressed as a bound on the KL-divergence between empirical and population errors\removed{, which holds for all predictors $h$ with a dependence on $\pi(h)$}.  This is a special case of the PAC-Bayes bound \citep[][Equation (4)]{mcallester2003simplified}, and is obtained directly by taking a union bound over a binomial tail bound\footnote{More specifically, by applying the binomial tail bound of \autoref{thm12} in \autoref{AppendixC} to each predictor $h$ in the support of $\pi$, with per-predictor failure probability $\delta_h=\pi(h)\delta/2$, and taking a union bound over all $h$.}:
\begin{equation}\label{pacybayes1}
\Pr_{S\sim\mathcal{D}^m}\left[ \; \forall_h \; \KL\left(L_S(h) \middle\Vert  L(h) \right) \leq \frac{\abs{h}_{\pi} + \log(\frac{m+1}{\delta/2})}{m} \; \right]  \geq 1-\delta/2.
\end{equation}
Focusing on $h=\MDL_\lambda(S)$, multiplying both sides of the inequality in \eqref{pacybayes1} by $\lambda$, and adding $H(L_S(\MDL_{\lambda}(S)))$ to both sides, we have that that with probability $\geq 1-\delta/2$,
\begin{align}
{\lambda}\KL(L_S(\MDL_{\lambda}(S)) \Vert& L(\MDL_{\lambda}(S))) + H(L_S(\MDL_{\lambda}(S))) \label{eqn8lhs}\\
&\leq H(L_S(\MDL_{\lambda}(S))) +{\lambda}\frac{\abs{\MDL_{\lambda}(S)}_{\pi}}{m} + {\lambda}\frac{\log(\frac{m+1}{\delta/2})}{m}\\
&\leq H(L_S(h^*)) + {\lambda}\frac{\abs{h^*}_{\pi}}{m} + {\lambda}\frac{\log(\frac{m+1}{\delta/2})}{m} + C\frac{\log m}{m} \label{ineqn10}\\
\intertext{and with probability $\geq 1-\delta$:}
&\leq H(L(h^*))+{\lambda}\frac{\abs{h^*}_{\pi}}{m} + {\lambda}\frac{\log(\frac{m+1}{\delta/2})}{m} + C'\sqrt{\frac{2(\log m)^2 \cdot \log \frac{1}{\delta/2}}{m}},\label{ineqn11}
\end{align}
for some constants $C, C'$. The second inequality \eqref{ineqn10} follows from the definition of $\MDL_{\lambda}$, and the last term $C\frac{\log m}{m}$ of \eqref{ineqn10} is the difference between the two objectives $J_{\lambda}$ and $\tilde{J}_{\lambda}$. In the third inequality \eqref{ineqn11} we bound the difference (with another failure probability of $\delta/2$) between the entropy of the empirical and population loss of the fixed predictor $h^*$ using McDiarmid's inequality.

We want to use this to get an upper bound on the population error $L(\MDL_\lambda(S))$.  The problem is that the left-hand-side \eqref{eqn8lhs} also depends on the empirical error $L_S(\MDL_\lambda(S))$, which we do not know and can't easily bound, except that by definition $L_S(\MDL_\lambda(S))\leq1/2$.  Instead, we'll replace this empirical error with $p=L_S(\MDL_\lambda(S))$ and minimize \eqref{eqn8lhs} w.r.t~$p$, as in $Q_{\lambda}(L(\MDL_{\lambda}(S))) = \min_{0 \leq p \leq 0.5} \lambda\KL(p \Vert L(\MDL_{\lambda}(S))) + H(p)$.  From the definition of this $Q_{\lambda}(q)$, we therefore have that $Q_{\lambda}(L(\MDL_{\lambda}(S)))$ is upper bounded by \eqref{eqn8lhs}, from which the Lemma follows.
\end{proof}

From \autoref{lemma8}, we can already see that as $m \rightarrow \infty$, $L(\MDL_{\lambda}(S)) \rightarrow Q^{-1}_{\lambda}(H(L(h^*)))=\ell_{\lambda}(L(h^*))$.  What is left is to simplify $Q^{-1}_{\lambda}(H(L(h^*)))$, and in order to obtain finite sample guarantees, also analyze applying $Q^{-1}_{\lambda}$ to the right-hand-side in \autoref{lemma8}.

\begin{proof} {\bf of \autoref{MDL UB} part (1), $0<\lambda\leq 1$:}

For $0 < \lambda \leq 1$ and $0 \leq q \leq 1/2$,  ${\lambda}\KL(p \Vert q) + H(p)$ is monotonically increasing in $p$, and thus the optimum is at $p^* = 0$. So in this case, $Q_{\lambda}(q) = - \lambda \log(1-q)$, and by \autoref{lemma8}, the limiting error is $Q_{\lambda}^{-1}(H(L(h^*))) = 1 - 2^{-\frac{1}{\lambda}H(L(h^*))}$. To get the finite sample guarantee, we use the inequality $ 1- 2^{-\alpha-A} \leq 1- 2^{-\alpha} + A$ (for $A,\alpha>0$) \citep[adapted from Lemma A.4 in][]{MS}.
\end{proof}
\vspace{-10pt}
\begin{proof} {\bf of \autoref{MDL UB} part (2), $1<\lambda$:}

When $1 < \lambda < \infty$, by taking the derivative of $\lambda \cdot \KL(p \Vert q) + H(p)$ w.r.t. $p$ and setting it to zero, we recover the minimizer $p^* = \frac{1}{1+ (\frac{1-q}{q})^{\frac{\lambda}{\lambda - 1}}}$. Plugging it in we have 
$Q_{\lambda}(q) = \lambda \cdot \KL\left(\frac{1}{1+ (\frac{1-q}{q})^{\frac{\lambda}{\lambda - 1}}}\middle\Vert q\right) + H\left(\frac{1}{1+ (\frac{1-q}{q})^{\frac{\lambda}{\lambda - 1}}}\right)=U_\lambda(q)$, and the limiting error is $U_\lambda^{-1}(H(L(h^*)))$.
To get the finite sample guarantee, we apply $U^{-1}_{\lambda}$ to both sides of \autoref{lemma8}, and then the mean value theorem on the right hand side.  When applying the mean value theorem, we bound the derivative of $U_{\lambda}^{-1}$ uniformly in terms of $L(h^*)$, which introduces the pre-factor of $1/(1-2L(h^*))$.  See details in \Cref{AppendixA.1}.
\end{proof}
\vspace{-10pt}
\begin{proof} {\bf of \autoref{lambda infty}, $1\ll\lambda$:}

As $\lambda\rightarrow\infty$, the first term inside the definition of $Q_\lambda$ (equation \eqref{eq:Q}) dominates, the minimizer is $p^*=q$, and we have $Q_\lambda(q)\rightarrow H(q)$.  We would therefore like to apply $H^{-1}$ to both sides of \autoref{lemma8} to obtain a bound on $L(\MDL(S))-L(h^*)$.  To do so for finite $\lambda$, we prove the following Lemma in Appendix \ref{AppendixA.2}, which quantifies how close $U_{\lambda}(q)$ is to the entropy function $H(q)$:
\begin{lemma}\label{lemma9} For any $\lambda >1$ and any $0 \leq q \leq \frac{1}{2}$, $H(q) < U_{\lambda}(q)+ \lambda/(\lambda-1)^2$.
\end{lemma}
Combining \autoref{lemma8} and \autoref{lemma9}, we have that with probability $\geq 1 - \delta$, 
\begin{equation}\label{thm3.4H}
    H(L(\MDL_{\lambda}(S))) \leq H(L(h^*)) + C\sqrt{\frac{2(\log m)^2 \cdot \log \frac{1}{\delta/2}}{m}}
    + {\lambda}\frac{\log(\frac{m+1}{\delta/2})}{m}+{\lambda}\frac{\abs{h^*}_{\pi}}{m} + \frac{\lambda}{(\lambda-1)^2}.
\end{equation}
We apply $H^{-1}$ to both sides of \eqref{thm3.4H}, use the mean value theorem, and bound the derivative of $H^{-1}$ by $\frac{\frac{1}{2} - L(h^*)}{1 - H(L(h^*))} \leq \frac{\ln 2}{1 - 2L(h^*)}$, yielding the desired result. See details in Appendix \ref{AppendixA.2}.\end{proof}

\removed{\begin{proof} {\bf of Corollary \ref{lambda_infty_coro},  Consistency when $1 \ll \lambda_m \ll m/\log m$.}

Since $1 \ll \lambda_m \ll m/\log m$ , as $m \rightarrow \infty$, all the terms inside the big-O notation on the right hand side of equation \eqref{thm3.4bd} in Theorem \ref{lambda infty} vanish, yielding the consistency result.
\end{proof}}

\section{Lower Bound Constructions and Proof Sketch}
In this Section, we describe  constructive lower bound proofs on the limiting error.  We show explicit constructions for $0 < \lambda < \infty$ ( \autoref{MDL LB}), for $\lambda_m \rightarrow 0$ or $\lambda = 0$ (\autoref{lambda0}), and for $\lambda_m = \Omega(m)$ with $\liminf \frac{\lambda_m}{m}> 10$ (\autoref{lambda>>m}).  In each regime, we construct specific hard learning problems, priors, and hypothesis classes such that the expected error of $\MDL_{\lambda}$ converges to the lower bound error asymptotically. Complete details and proofs can be found in \autoref{AppendixB}.

\subsection{Lower Bound for $0 < \lambda < \infty$ (proof of \autoref{MDL LB})}\label{sec:finite LB proof}

For any $0<\lambda<\infty$, any $0<L^*<0.5$, and any $L^* \leq L' < \ell_\lambda(L^*)$, we will construct a source distribution (hard learning problem) $D$ and a prior $\pi$, and show a hypothesis $h^*$ with $\pi(h^*) \geq 0.1$ and $L_D(h^*) = L^*$, such that $\mathbb{E}_S \left[L_D(\MDL_{\lambda}(S))\right] \rightarrow L'$ as the sample size increases ($m \rightarrow \infty$).  

Specifically, we will construct a distribution $D$ over infinite binary sequences $x = x[0] x[1]... \in \mathcal{X} = \{0,1\}^\infty$ and binary labels $y\in\{\pm1\}$, and a prior over hypothesis $h_i(x)=x[i]$ with\footnote{This is a simple and explicit ``universal'' prior, in the sense that $\abs{h_i}_\pi=\log i+O(\log\log i)$, and it ensures $\pi(h_0)=0.1$ (we treat $0\cdot \log^2 0 = 0$).} $\pi(h_i)=1/(i \cdot \log^2 i + 10)$, where each hypothesis is based on one bit of the input (this just allows us to directly specify the joint distribution over the behavior of the hypothesis by specifying the distribution of $x$).  In our constructions $h_0(x)=x[0]$ will always be the ``good'' predictor, $h^*=h_0$, with low population error $L_D(h_0)=\Pr[x[0]\neq y]=L^*$, while all $h_i$, $i\geq 1$, will be ``bad'', with $L_D(h_i)=L'>L^*$.  We will ensure that as $m\rightarrow\infty$, $\MDL_\lambda$ will select one of these ``bad'' predictors, i.e.~$\Pr_{S\sim D^m}\left[\MDL_\lambda(S)=h_0\right]\xrightarrow{m\rightarrow\infty} 0$ and $L(\MDL_\lambda(S))\xrightarrow{p} L'$.

Given $L^*,L'$, we consider a source distribution $D$ where $y = \text{Ber}(\frac{1}{2})$, and each bit $x[i]$ is independent conditioned on $y$, with $x[0] = y \oplus \text{Ber}(L^*)$, while $x[i] = y \oplus \text{Ber}(L')$. This ensures $L_D(h_0)=L^*$ while $L_D(h_i)=L'$ for $i\geq 1$.

We will analyze the $\MDL$ objective $J_\lambda(h,S)$, or rather its approximation $\tilde{J}_\lambda(h,S)=\lambda \abs{h}_\pi + m H(L_S(h))$ (as in equation \eqref{stirling}). We will argue that (with probability approaching one), $J_\lambda(h,S)$ is minimized not on $h_0$, and hence $\MDL_\lambda(S)=h_i$ for $i\geq 1$ and so $L(\MDL_\lambda(S))=L'$. For the ``good'' predictor $h_0$ we have that $L_S(h_0)\xrightarrow{p} L(h_0)=L^*$, and hence $\tilde{J}_\lambda(h_0, S) \xrightarrow{p} mH(L^*) + \lambda \log 10 $. For an explicit function $k(m)$, we will show that, with probability approaching one, there exists $1\leq i\leq k(m)$ with $\tilde{J}_\lambda(h_i,S) < \tilde{J}_\lambda(h_0,S)- \Omega(m) < \tilde{J}_\lambda(h_0,S)- \omega(\log m)$, ensuring $h_0$ does not minimize $J_\lambda(h,S)$ (the $\omega(\log m)$ gap ensures that the difference between $J$ and $\tilde{J}$ is insignificant compared to the gap).

\begin{enumerate}
\item $\lambda \leq 1$: Take $k(m) = \frac{2\sqrt{m}}{(1-L')^m}$, then (with probability approaching one), there exists some ``bad'' classifier $h_{\hat{i}}$ with $1\leq \hat{i} \leq k(m)$ such that $L_S(h_{\hat{i}}) = 0$, and so
\begin{align}
 \tilde{J}_{\lambda}(h_{\hat{i}},S) &=\lambda \cdot (\log \hat{i} + O(\log\log \hat{i})) + m H(0) 
 \leq \lambda \log k(m) + O(\log\log k(m)) + 0\\
 &=\lambda (1+\tfrac{1}{2}\log m - m \log(1-L')) + O(\log\log k(m)) \\
 &\leq  - \lambda m \log(1-L') +  O(\log m) 
 < m H(L^*) - \Omega(m) = \tilde{J}_\lambda(h_0,S) - \Omega(m)
\end{align}
where in the final inequality we used $L' < 1 - 2^{-H(L^*)/{\lambda}}$, and the asymptotic notation is w.r.t.~$m\rightarrow\infty$.

\item $\lambda>1$: Take $k(m) = 2^{m\KL(\hat{L}\Vert L')}$ where $\hat{L} = \frac{1}{1+ (\frac{1-L'}{L'})^\frac{\lambda}{\lambda - 1}}$. Let $h_{\hat{i}}$ be the empirical error minimizer among the first $k(m)$ bad predictors, i.e.~such that $L_S(h_{\hat{i}}) = \min_{i=1\ldots k(m)} L_S(h_i)$.  This is the minimum of $k(m)$ independent (scaled) binomials $\text{Bin}(m,L')$, and so concentrates (see \autoref{thm12} in \autoref{AppendixC}) s.t.~$\KL( L_S(h_{\hat{i}}) \Vert L') \xrightarrow{p} \frac{\log k(m)}{m}=\KL(\hat{L}\Vert L')$, and hence $L_S(h_{\hat{i}}) \xrightarrow{p} \hat{L}$ and 
\begin{align}
   \tilde{J}_{\lambda}(h_{\hat{i}},S) &\xrightarrow{p} \lambda \cdot (\log \hat{i} + O(\log\log \hat{i})) + m H(L_S(h_{\hat{i}})) \\
& \leq \lambda \log k(m) + O(\log\log k(m)) + m H(\hat{L}) + o( m)\label{eq:Jbad0} \\
 &= m\left(\lambda \KL(\hat{L}\Vert L')+H(\hat{L})\right) + o(m)
 = m U_\lambda(L') + o(m) \label{eq:Jbad1}\\
 &< m U_\lambda(U^{-1}_\lambda(H(L^*))) - \Omega(m) = m H(L^*)-\Omega(m)=\tilde{J}_\lambda(h_0,S) - \Omega(m)\label{eq:Jbad2}
\end{align}
where in \eqref{eq:Jbad1} we plugged in $k(m)$ and used the definition of $U_\lambda$ from equation \eqref{ell}, and in \eqref{eq:Jbad2} we relied on  $L' < U_{\lambda}^{-1} ( H(L^*) )$. See further explanations in \Cref{AppendixB.1}. \qed
\end{enumerate}

\subsection{Lower Bound for $\lambda_m \rightarrow 0$ or $\lambda = 0$ (proof of \autoref{lambda0})}
We now turn to $\lambda_m \rightarrow 0$ or $\lambda = 0$, and show that for any $0<L^* <0.5$ and $L^* \leq L' < 1$, the source distribution described in \autoref{sec:finite LB proof}, and with the same prior, such that $L(\MDL_{\lambda}(S)) \xrightarrow{p} L'$ as $m \rightarrow \infty$ despite $L(h_0)=L^*$ and $\pi(h_0)=0.1$.

If $\lambda = 0$, then $\MDL_{\lambda}$ simply minimizes $L_S(h)$.  There exists a.s.~some $\hat{i}$ with $L_S(h_{\hat{i}})=0$, but on the other hand $L_S(h_0)\xrightarrow{p}L^*>0$.  Hence, with probability approaching one, $\MDL_0(S)\neq h_0$ and so $L(\MDL_0(S))=L'$.

If $\lambda_m \rightarrow 0$ as $m \rightarrow \infty$, let $\hat{i}$ denote the smallest index $\hat{i}\geq 1$ such that $L_S(h_{\hat{i}}) = 0$. We already saw that $\hat{i} \leq\frac{m+1}{(1-L')^m}$ with probability approaching one. We therefore have that with probability approaching one, $\tilde{J}_{\lambda_m}(h_{\hat{i}},S) = \lambda_m \abs{h_{\hat{i}}}+mH(0)\leq\lambda_m \log\frac{m+1}{(1-L')^m}=o(m)$, where in the last step we used $\lambda_m\rightarrow 0$.  On the other hand, $\tilde{J}_{\lambda_m}(h_0,S)\xrightarrow{p} m H(L^*)=\Omega(m)$.  See details in \Cref{AppendixB.2}. \qed

\subsection{Lower Bound for $\lambda_m = \Omega(m)$ with $\liminf \frac{\lambda_m}{m}> 10$ (proof of \autoref{lambda>>m})}
We now turn to $\lambda_m = \Omega(m)$ with $\liminf \frac{\lambda_m}{m}> 10$, and show that for any $0 \leq L^* <0.5$ and $L^* \leq L' < 0.5$, the source distribution described in \autoref{sec:finite LB proof} with only two predictors $\{h_0, h_1\}$, $L(h_0)=L^*,L(h_1)=L'$, and with the prior $\pi(h_0)=0.1$ and $\pi(h_1)= 0.9$, such that $L(\MDL_{\lambda}(S))\xrightarrow{p} L'$ as $m \rightarrow \infty$. 

Since $L_S(h_0)\xrightarrow{p}L^*$ and $L_S(h_1)\xrightarrow{p}L'$, we have that %$L_S(h_1)\xrightarrow{p} L(h_1)=L'$, and hence
% such that $h_0(x)=x$, $h_1(x)=1$, and a source distribution where 
 % Given $L^* \in [0, \frac{1}{2})$, consider a source distribution $D$ where $y = 0$, and $x = y \oplus \text{Ber}(L^*)$. Consider the prior $\pi(h_0)=0.1$ and $\pi(h_1)= 0.9$. It is easy to check that $L(h_0) = L^*$ and $L(h_1) = 1$.
% Given this learning problem and that $\liminf \frac{\lambda_m}{m}> 10$, it is easy to see that with probability one, $J(h_1) < J(h_0)$ for all large $m$, so $\MDL_{\lambda_m}$ outputs $h_1$ instead of $h_0$ for all large $m$. The expectation bound then follows. See details in \Cref{AppendixB.3}.
% \end{hproof}
\begin{align}
        \tilde{J}_{\lambda_m}(h_1, S)&=\lambda_m\log\frac{10}{9} + mH(L_S(h_1))
        \xrightarrow{p}\lambda_m\log\frac{10}{9} + mH(L')\\
        &<\lambda_m\log\frac{10}{9}+ m + mH(L^*)+\Omega(m)\notag\\
        &<\lambda_m\log10 + mH(L^*)+\Omega(m) = \tilde{J}_{\lambda_m}(h_0, S)+\Omega(m)
\end{align}
where in the final inequality we used $\liminf \frac{\lambda_m}{m}> 10$. See details in \Cref{AppendixB.3}.

\section{Contrast with Well-Specified Case}\label{wellspecifiedsection}
It is interesting to contrast the agnostic setting studied above to a well-specified setting, where the noise is a result of random label noise. Formally, a source distribution $D$ is well specified if $Y|X = Y|h^*(X)$
(that is, $Y\perp X|h^*(X)$), which means that  
$Y|X = h^*(X) \oplus \text{Ber}(L^*)$ for some Bayes optimal predictor $h^*$ and independent Bernoulli noise.  Note that this condition is not satisfied in any of the hard problem constructions of our lower bound proofs. In other words, all the source distributions in our lower bound proofs are mis-specified. 
In fact, in the well-specified case, as already noted by \citet{GL}, $\lambda=1$ leads to asymptotic consistency, following the classical analysis of MDL \citep{BC}. However, as is well understood in the MDL literature \cite[e.g.][]{TongZ}, this consistency does not enjoy a uniform rate or finite sample guarantee. In our language, it provides an upper bound on the left-side expression in \autoref{cor:finite}, where we take the limit $m\rightarrow\infty$ separately for each prior $\pi$ and source $D$, but not the right-side expression where we first take the limit $m\rightarrow\infty$.  
For the right-side expression in \autoref{cor:finite}, even in the well-specified case and with $\lambda=1$, we can obtain an upper bound of $2L^*(1-L^*)>L^*$\footnote{This uniform upper bound can be obtained from the weak convergence result from \citet{TongZ} by choosing a particular test function $f(x,y) = \mathbb{P}_{y'|x \sim p_{h^*, L^*}}(y \neq y'|x)$.}, and also show that this ``uniform'' limiting error is strictly larger than $L^*$. In this sense, we have tempered behavior, with a better tempering function $2L^*(1-L^*)<\ell_1(L^*)$ than the agnostic case we focus on in this paper. It would be interesting to understand this problem further: what is best uniform rate with $\lambda=1$?  Is this tempering function tight? What is the uniform and non-uniform limiting error when $\lambda>1$, and with $\lambda<1$ ?  Is there a discontinuity at $\lambda=1$?

\section{Summary and Discussion}\label{summary section}
In this paper, we provided a tight analysis, with matching upper bounds and worst-case lower bounds, on the limiting error of $\MDL_{\lambda}$, for any $0<\lambda<\infty$. This improves both the lower and upper bounds over \citet{GL} for the special case $\lambda=1$, and generalizes to any $\lambda$. 

We also characterize the behavior as $\lambda\rightarrow0$ and $\lambda\rightarrow\infty$, with a gap between $\lambda = \Theta(m/\log(m))$ and $\lambda = \Theta(m)$. This $\log$-factor comes from the $\log$ factor in the Binomial tail bound (see \autoref{AppendixC}), which also appears in all PAC-Bayes bounds and in many SRM-type bounds based on $\log \pi$.  We do not know if this log-factor can be avoided, and it could be interesting to characterize the fine grained behavior at this transition.

Our analysis does not assume any structure on the prior, and so can be thought of as the ``baseline'' or absolute worst case overfitting behavior. For many specific priors, and perhaps for special classes of source distributions, we know that even with $\lambda=0$ one can get tempered, or even benign behavior.  This work can serve as a basis for understanding overfitting for specific priors.

\paragraph{Acknowledgments} This work was initiated by an anonymous reviewer who directed us to \citet{GL}. We would like to thank Naren Manoj for helpful discussions and Mesrob I. Ohannessian for working with us on a crisp formulation of \autoref{AppendixC}. This work was done as part of the NSF-Simons Collaboration on the Mathematics of Deep Learning and the NSF TRIPOD Institute on Data Economics Algorithms and Learning.

\bibliographystyle{plainnat}
\bibliography{Reference}

\newpage
\appendix

\section{Generalization Guarantees and Proof of Upper Bounds}\label{AppendixA}
In this section, we provide proofs for \autoref{MDL UB} and \autoref{lambda infty}. We first prove an important lemma, \autoref{lemma8}, which serves as the basis of the proof of both theorems.

% Restating Lemma
\restate{Lemma}{lemma8} % Define the name dynamically
\begin{restatethm}
For some constant $C$, any $0 <\lambda < \infty$, with probability $1- \delta$ over $S\sim\mathcal{D}^m$, for any predictor $h^*$:
\begin{gather}
    Q_{\lambda}(L(\MDL_{\lambda}(S))) \leq H(L(h^*)) + {\lambda}\frac{\log(\frac{m+1}{\delta/2})}{m}+{\lambda}\frac{\abs{h^*}_{\pi}}{m} + C\sqrt{\frac{2(\log m)^2 \cdot \log \frac{1}{\delta/2}}{m}} \notag\\
\textrm{where:}\quad\quad Q_{\lambda}(q) = \min_{0 \leq p \leq 0.5}{\lambda}\KL(p \Vert q) + H(p) \label{eq:Q}
\end{gather}
\end{restatethm}

\begin{proof}
We start from a concentration guarantee, expressed as a bound on the KL-divergence between empirical and population errors\removed{, which holds for all predictors $h$ with a dependence on $\pi(h)$}.  This is a special case of the PAC-Bayes bound \citep[][Equation (4)]{mcallester2003simplified}, and is obtained directly by taking a union bound over a binomial tail bound\footnote{More specifically, by applying the binomial tail bound of \autoref{thm12} in \autoref{AppendixC} to each predictor $h$ in the support of $\pi$, with per-predictor failure probability $\delta_h=\pi(h)\delta/2$, and taking a union bound over all $h$.}:
\begin{equation}\label{pacybayes2}
\Pr_{S\sim\mathcal{D}^m}\left[ \; \forall_h \; \KL\left(L_S(h) \middle\Vert  L(h) \right) \leq \frac{\abs{h}_{\pi} + \log(\frac{m+1}{\delta/2})}{m} \; \right]  \geq 1-\delta/2.
\end{equation}
Focusing on $h=\MDL_\lambda(S)$, multiplying both sides of the inequality in \eqref{pacybayes2} by $\lambda$, and adding $H(L_S(\MDL_{\lambda}(S)))$ to both sides, we have that that with probability $\geq 1-\delta/2$,
\begin{align}
{\lambda}\KL(L_S(\MDL_{\lambda}(S)) \Vert& L(\MDL_{\lambda}(S))) + H(L_S(\MDL_{\lambda}(S))) \label{eqn21lhs}\\
&\leq H(L_S(\MDL_{\lambda}(S))) +{\lambda}\frac{\abs{\MDL_{\lambda}(S)}_{\pi}}{m} + {\lambda}\frac{\log(\frac{m+1}{\delta/2})}{m}\\
&\leq H(L_S(h^*)) + {\lambda}\frac{\abs{h^*}_{\pi}}{m} + {\lambda}\frac{\log(\frac{m+1}{\delta/2})}{m} + C\frac{\log m}{m} \label{ineqn23}
\end{align}
for some constants $C, C'$. The second inequality \eqref{ineqn23} follows from the definition of $\MDL_{\lambda}$, and the last term $C\frac{\log m}{m}$ of \eqref{ineqn23} is the difference between the MDL objective $J_{\lambda}$ and its approximate form $\tilde{J}_{\lambda}$ as defined in \eqref{stirling}. 

Note that $H(L_S(h^*))$ concentrates to its expectation. Observe that even though $H$ is not Lipschitz, it's still the case that $|H(p+q)-H(p)| \leq H(q) \leq 2q\log(1/q)$ for $q<\frac{1}{2}$, and changing a single sample in $S$ can only change $L_S(h^*)$ by at most $1/m$, and so $H(L_S(h^*))$ by at most $2\log(m)/m$. In this way, for any $h$, the function  $S \rightarrow H(L_S(h^*))$ satisfies the bounded difference property with differences $c_i = 2\log(m)/m$. Therefore, by McDiarmid's inequality, $H(L_S(h^*))$ concentrates:
\begin{align}\label{mcdiarmid}
    \mathbb{E}\left[H(L_S(h^*))\right] > H(L_S(h^*))
    - \sqrt{\frac{2(\log m)^2 \cdot \log \frac{1}{\delta/2}}{m}}, 
\end{align}
with probability $\geq 1-\delta/2$.

Therefore, combining the two high probability events \eqref{ineqn23} and \eqref{mcdiarmid} using the union bound, we get with probability $\geq 1- \delta$,
\begin{align}
    {\lambda}\KL(L_S(\MDL_{\lambda}(S)) &\Vert L(\MDL_{\lambda}(S))) + H(L_S(\MDL_{\lambda}(S)))\label{lemma5.1highprob}\\
    &\leq \mathbb{E}\left[H(L_S(h^*))\right] +{\lambda}\frac{\abs{h^*}_{\pi}}{m} + {\lambda}\frac{\log(\frac{m+1}{\delta/2})}{m} + C\frac{\log m}{m} + \sqrt{\frac{2(\log m)^2 \cdot \log \frac{1}{\delta/2}}{m}}\\
    &\leq H(L(h^*))+{\lambda}\frac{\abs{h^*}_{\pi}}{m} + {\lambda}\frac{\log(\frac{m+1}{\delta/2})}{m} + C'\sqrt{\frac{2(\log m)^2 \cdot \log \frac{1}{\delta/2}}{m}},
\end{align}
for some constant $C'$. In the second inequality, we use Jenson's inequality $\mathbb{E}[H(L_S(h^*))] \leq H(\mathbb{E}[L_S(h^*)]) = H(L(h^*))$.

We want to use this to get an upper bound on the population error $L(\MDL_\lambda(S))$.  The problem is that the left-hand-side \eqref{lemma5.1highprob} also depends on the empirical error $L_S(\MDL_\lambda(S))$, which we do not know and can't easily bound, except that by definition $L_S(\MDL_\lambda(S))\leq1/2$.  Instead, we'll replace this empirical error with $p=L_S(\MDL_\lambda(S))$ and minimize \eqref{lemma5.1highprob} w.r.t~$p$, as in $Q_{\lambda}(q) = \min_{0 \leq p \leq 0.5} \lambda\KL(p \Vert q) + H(p)$.  From the definition of this $Q_{\lambda}(q)$, we therefore have that $Q_{\lambda}(L(\MDL_{\lambda}(S)))$ is upper bounded by \eqref{lemma5.1highprob}, from which the Lemma follows.
\end{proof}
From \autoref{lemma8}, we can already see that as $m \rightarrow \infty$, $L(\MDL_{\lambda}(S)) \rightarrow Q^{-1}_{\lambda}(H(L(h^*)))=\ell_{\lambda}(L(h^*))$.  The proof of \autoref{MDL UB} and \autoref{lambda infty} then reduces to simplifying $Q^{-1}_{\lambda}(H(L(h^*)))$, and also applying $Q^{-1}_{\lambda}$ to the right-hand-side in \autoref{lemma8}. To analyze $Q_{\lambda}$, we need to optimize over $p \in [0,0.5]$. It turns out the minimum point $p^*$ is different depending on the value/scaling of $\lambda$. 

\subsection{Proof of \autoref{MDL UB} ($0 < \lambda < \infty$)}\label{AppendixA.1}
Consider the function $\ell_{\lambda}$:
\[
\ell_{\lambda}(L^*)= Q_{\lambda}^{-1}(H(L^*))=
\begin{cases}
1 - 2^{-\frac{1}{\lambda}H(L^*)},  & \text{for } 0< \lambda \leq 1 \\
U_{\lambda}^{-1}(H(L^*)), & \text{for } \lambda > 1,
\end{cases}
\]
where $Q_{\lambda}(q) = \min_{0 \leq p \leq 0.5}{\lambda}\KL(p \Vert q) + H(p)$, and $U_{\lambda}(q) = \lambda \KL(\frac{1}{1+ (\frac{1-q}{q})^{\frac{\lambda}{\lambda - 1}}}\Vert q) + H(\frac{1}{1+ (\frac{1-q}{q})^{\frac{\lambda}{\lambda - 1}}})$.\\
\restate{Theorem}{MDL UB} % Define the name dynamically
\begin{restatethm} [Agnostic Upper Bound]
(1) For any $0<\lambda\leq 1$,  any source distribution $D$, any predictor $h^*$, any valid prior $\pi$, and any $m$:
\begin{align}
    \underset{S \sim D^m}{\mathbb{E}}[L(\MDL_{\lambda}(S))]
    \leq 1 - 2^{-\frac{1}{\lambda}H(L(h^*))} + O\left(\frac{\abs{h^*}_{\pi}}{m} + \frac{1}{\lambda}\sqrt{\frac{\log^3 (m)}{m}}\right).
\end{align}
(2)
For any $\lambda > 1$,  any source distribution $D$, any predictor $h^*$, any valid prior $\pi$, and any $m$:
\begin{align}
    \underset{S \sim D^m}{\mathbb{E}}[L(\MDL_{\lambda}(S))]
    \leq &U_{\lambda}^{-1}(H(L(h^*)))
    + O\left(\frac{1}{(1 - 2L(h^*))^2} \cdot \left(\lambda\left(\frac{\abs{h^*}_{\pi} + \log m}{m}\right) + \sqrt{\frac{\log^3 (m)}{m}}\right)\right).
\end{align}
Where $O(\cdot)$ only hides an absolute constant, that does not depend on $D, \pi$ or anything else.
\end{restatethm}

\begin{proof}
For $0 < \lambda \leq 1$ and $0 \leq q \leq 1/2$, it is easy to check that the derivative of ${\lambda}\KL(p \Vert q) + H(p)$ w.r.t. $p$ is non-negative, which means it is monotonically increasing, and thus the optimum is at $p^* = 0$. So in this case, $Q_{\lambda}(q) = - \lambda \log(1-q)$.\\
Plugging $Q_{\lambda}(q) = - \lambda \log(1-q)$ into the \autoref{lemma8}, we have for some constant $C$, with probability $\geq 1 - \delta$, 
\begin{align}
    -\lambda \log(1-L(\MDL_{\lambda}(S))) \leq H(L(h^*)) + C\sqrt{\frac{2(\log m)^2 \cdot \log \frac{1}{\delta/2}}{m}} + {\lambda}\frac{\log(\frac{m+1}{\delta/2})}{m}+{\lambda}\frac{\abs{h^*}_{\pi}}{m}.  
\end{align}
Hence, with probability $\geq 1 - \delta$, 
\begin{align}
   L(\MDL_{\lambda}(S)) &\leq 1 - 2^{-\frac{H(L(h^*))}{\lambda} - \left(\frac{C}{\lambda}\sqrt{\frac{2(\log m)^2 \cdot \log \frac{1}{\delta/2}}{m}} + \frac{\log(\frac{m+1}{\delta/2})}{m}\right) - \left(\frac{\abs{h^*}_{\pi}}{m}\right)}\\
   &\leq 1-2^{-\frac{H(L(h^*))}{\lambda}}+ \left(\frac{C}{\lambda}\sqrt{\frac{2(\log m)^2 \cdot \log \frac{1}{\delta/2}}{m}} + \frac{\log(\frac{m+1}{\delta/2})}{m}\right) +\frac{\abs{h^*}_{\pi}}{m},\label{lambda<1:ineq}
\end{align}
where in \eqref{lambda<1:ineq}, we use the inequality 
\begin{align}\label{eqn9}
\text{For any } \alpha, A \geq 0, 1- 2^{-\alpha-A} \leq 1- 2^{-\alpha} + A,
\end{align}
which is adapted from Lemma A.4 in \citet{MS}.

Since the risk is bounded, the high probability bound implies the bound on expected risk:
\begin{equation}
    \mathbb{E}L(\MDL_{\lambda}(S)) \leq 1-2^{-\frac{H(L(h^*))}{\lambda}}
    + \left(\frac{C}{\lambda}\sqrt{\frac{2(\log m)^2 \cdot \log \frac{1}{\delta/2}}{m}} + \frac{\log(\frac{m+1}{\delta/2})}{m}\right)
    +\frac{\abs{h^*}_{\pi}}{m} + \delta.
\end{equation}
Take $\delta=\frac{1}{\sqrt{m}}$, given $0 < \lambda \leq 1$, this gives us
\begin{equation}
\mathbb{E}[L(\MDL_{\lambda}(S))] \leq  1 - 2^{-\frac{1}{\lambda}H(L(h^*))} + O\left(\frac{\abs{h^*}_{\pi}}{m} + \frac{1}{\lambda}\sqrt{\frac{\log^3 (m)}{m}}\right).
\end{equation} 
This concludes the proof for $0 < \lambda \leq 1$.

On the other hand, when $1 < \lambda < \infty$, the minimum point $p^*$ is not always at zero. Taking the derivative of $\lambda \cdot \KL(p \Vert q) + H(p)$ w.r.t. $p$, and setting it to zero, we get $p^* = \frac{1}{1+ (\frac{1-q}{q})^{\frac{\lambda}{\lambda - 1}}}$. So in this case, 
\begin{align}
    Q_{\lambda}(q) &=  \min_{0 \leq p \leq 0.5}{\lambda}\KL(p \Vert q) + H(p) \\
    &= \lambda \cdot \KL\left(\frac{1}{1+ (\frac{1-q}{q})^{\frac{\lambda}{\lambda - 1}}}\middle\Vert q\right) + H\left(\frac{1}{1+ (\frac{1-q}{q})^{\frac{\lambda}{\lambda - 1}}}\right) = U_{\lambda}(q).\notag
\end{align}
% Denote this minimum function as $U_{\lambda}$ such that $U_{\lambda}(q) =  \lambda \cdot \KL\left(\frac{1}{1+ (\frac{1-q}{q})^{\frac{\lambda}{\lambda - 1}}}\middle\Vert q\right) + H\left(\frac{1}{1+ (\frac{1-q}{q})^{\frac{\lambda}{\lambda - 1}}}\right)$.

Plugging $Q_{\lambda}(q) = U_{\lambda}(q)$ into \autoref{lemma8}, we have for some constant $C$, with probability $\geq 1 - \delta$,  
\begin{equation}\label{lambda>1hpb}
    U_{\lambda}(L(\MDL_{\lambda}(S))) \leq H(L(h^*)) + C\sqrt{\frac{2(\log m)^2 \cdot \log \frac{1}{\delta/2}}{m}}
    + {\lambda}\frac{\log(\frac{m+1}{\delta/2})}{m}+{\lambda}\frac{\abs{h^*}_{\pi}}{m}.
\end{equation}
Taking $\delta = \frac{1}{\sqrt{m}}$, we have with probability $\geq 1 - \frac{1}{\sqrt{m}}$,  for some constant $C', C''$,
\begin{equation}\label{eqnUlambda}
    U_{\lambda}(L(\MDL_{\lambda}(S))) \leq H(L(h^*)) + C'\frac{(\log m)^\frac{3}{2}}{\sqrt{m}}
    + {\lambda}C''\frac{\log m}{m}+\lambda \frac{\abs{h^*}_{\pi}}{m}.
\end{equation}
    
    Let $\Delta = C'\frac{(\log m)^\frac{3}{2}}{\sqrt{m}}
    + {\lambda}C''\frac{\log m}{m}+{\lambda}\frac{\abs{h^*}_{\pi}}{m}$. 
    % For all large enough sample size $m \gtrsim \max\{\frac{2\lambda}{1 - H(L(h^*))}\log^2(\frac{2\lambda}{1 - H(L(h^*))}), \frac{4}{(1 - H(L(h^*)))^2}\log^3(\frac{4}{(1 - H(L(h^*)))^2})\}$, we have 
    If $\Delta < \frac{1}{2}(1 - H(L(h^*)))$ and so the right hand side of \eqref{eqnUlambda} $H(L(h^*)) + \Delta< \frac{1+H(L(h^*))}{2} < 1$, it is then well-defined to apply the inverse function $U_{\lambda}^{-1}$ on both sides of \eqref{eqnUlambda} to yield that with probability $\geq 1 - \frac{1}{\sqrt{m}}$, 
\begin{align}
    L(\MDL_{\lambda}(S)) &\leq U_{\lambda}^{-1}\left(H(L(h^*)) + \Delta \right).    
\end{align}
Since the risk is bounded, the high probability bound implies the bound on expected risk:
\begin{equation}\label{eqn_15}
    \mathbb{E}L(\MDL_{\lambda}(S)) \leq U_{\lambda}^{-1}\left(H(L(h^*)) + \Delta \right) + \frac{1}{\sqrt{m}}.
\end{equation}
By the mean value theorem
    \begin{align}
        U_{\lambda}^{-1}( H(L(h^*)) + \Delta ) &=U_{\lambda}^{-1}( H(L(h^*)) ) + (U_{\lambda}^{-1})'(\xi)\Delta \\
        &= U_{\lambda}^{-1}( H(L(h^*)) ) + \frac{1}{U_{\lambda}'(U_{\lambda}^{-1}(\xi))}\Delta, \label{eqn__20}
    \end{align}
for some $\xi \in (H(L(h^*)), H(L(h^*)) + \Delta)$.

Since $H(L(h^*)) + \Delta < \frac{1+H(L(h^*))}{2} < 1$,  $\xi$ lies strictly inside a sub-interval of $(0,1)$ and bounded away from $0$ and $1$. The following lemma shows that then $\frac{1}{U_{\lambda}'(U_{\lambda}^{-1}(\xi))}$ is uniformly (over all $\lambda >1$) upper bounded by some function depending on $L(h^*)$. 

\begin{lemma}\label{lemmaA.1} For some positive constant $c >0$, any $\lambda > 1$, any $L^* \in (0, 0.5)$, and any $\xi \in \left(H(L^*), \frac{1+H(L^*)}{2}\right)$:
\begin{gather}
    \frac{1}{U_{\lambda}'(U_{\lambda}^{-1}(\xi))} \leq \frac{1}{\min\left(c, H'\left(1 - 2^{-\frac{H(L^*) + 1}{2}}\right)\right)} = O\left(\frac{1}{(L^* - \frac{1}{2})^2}\right)
\end{gather}
\end{lemma}
\begin{proof}{\bf of \autoref{lemmaA.1}:}
    It is equivalent to proving $\forall \lambda > 1$, $\forall L^* \in (0, 0.5)$, $\forall \xi \in (H(L^*), \frac{1+H(L^*)}{2})$, we have $U_{\lambda}'(U_{\lambda}^{-1}(\xi)) \geq \min\left(c, H'\left(1 - 2^{-\frac{H(L^*) + 1}{2}}\right)\right)$. To prove the statement, we split into two cases: $L^* \leq 0.45$, and $L^* > 0.45$, and show that the derivative $U_{\lambda}'(U_{\lambda}^{-1}(\xi))$ is uniformly lower bounded in each case.

    Case (1): When $L^* < 0.45$, then $\xi \in \left(H(L^*), \frac{1+H(L^*)}{2}\right) < \frac{1 + H(L^*)}{2} < 0.997$. We will show the derivatives $U_{\lambda}'(U_{\lambda}^{-1}(\xi))$ for all $\lambda >1$ and $\xi <0.997$ stay away from 0. Indeed, we can find a positive constant $c > 0$ such that for any $\lambda> 1$, and $\xi < 0.997$, $U_{\lambda}'(U_{\lambda}^{-1}(\xi)) \geq c$:
    
    By Envelope Theorem, we can find the derivative of $U_{\lambda}$ to be
    \begin{align}\label{U_deriv}
        U_{\lambda}'(q) = \frac{\lambda}{\ln 2}\left[\frac{1 - p^*(q)}{1 - q} - \frac{p^*(q)}{q}\right],
    \end{align}
    where $p^*(q) = \frac{1}{1 + (\frac{1 - q}{q})^{\frac{\lambda}{\lambda - 1}}}$ is the minimizer of $U_{\lambda}$ for $q \in (0, 0.5)$. Observe that $p^*(q) < q$ for $q \in (0, 0.5)$. By Taylor expansion of $p^*$, we have $p^* \rightarrow q$, $U_{\lambda} \rightarrow H$ and $U'_{\lambda} \rightarrow H'$ pointwise, as $\lambda \rightarrow \infty$.
    For each $\xi <0.997$, $U_{\lambda}^{-1}(\xi)$ and $H^{-1}(\xi)$ stays within $(0, 0.5)$, so we have $U_{\lambda}'(U_{\lambda}^{-1}(\xi)) \rightarrow H'(H^{-1}(\xi))$ due to the monotonicity and continuity of $U_{\lambda}$ and $H$. Because the domain $[0, 0.997]$ for $\xi$ is compact and $U'_{\lambda}(U_{\lambda}^{-1}(\xi))$ and $H'(H^{-1}(\xi))$ are both continuous in $\xi$, we can conclude uniform convergence such that
    for any fixed $\epsilon >0$, we can find a $\lambda_0$ such that for all $\lambda > \lambda_0$ and all $\xi \leq 0.997$, $\left|U_{\lambda}'(U_{\lambda}^{-1}(\xi)) - H'(H^{-1}(\xi))\right| < \epsilon$. Taking $\epsilon = \frac{1}{2}\min_{\xi \leq 0.997}H'(H^{-1}(\xi))$ yields that $U_{\lambda}'(U_{\lambda}^{-1}(\xi)) > \frac{1}{2}\min_{\xi \leq 0.997}H'(H^{-1}(\xi)) = 0.093$ for all $\lambda > \lambda_0$ and all $\xi \leq 0.997$. 
    % Denote this value as $c_1 = \frac{1}{2}\min_{\xi \leq 0.997}H'(H^{-1}(\xi))$, and note that $c_1 > 0$ since 0.997 is bounded away from the endpoint 1. 

    On the other hand, because the function $(\xi, \lambda) \mapsto U_{\lambda}'(U_{\lambda}^{-1}(\xi))$ is continuous over the compact domain $[0, 0.997] \times [1, \lambda_0]$ (where we define $U_1(q) = -\log(1-q)$), by extreme value theorem and that $U_{\lambda}'(U_{\lambda}^{-1}(\xi))>0$ over this domain, $U_{\lambda}'(U_{\lambda}^{-1}(\xi))$ achieves a strictly positive minimum on this set and denote this minimum as $c_0$.

    Let $c = \min (0.093, c_0)$, which is thus the uniform positive lower bound we found for $U_{\lambda}'(U_{\lambda}^{-1}(\xi))$, for all $\lambda > 1$ and $\xi < 0.997$. 
    % Indeed, we numerically check $c \approx 0.204$.

    Case (2): 
    % and $U_{\lambda}^{-1}(\xi) \leq U'_1(\frac{1+H(L^*)}{2}) = 1 - 2^{-\frac{1+H(L^*)}{2}} < 0.5$.
    % \[
    % \begin{split}
    %     U'_{\lambda}(q) = \frac{\lambda}{\ln 2}\left[\frac{q - p^*(q)}{q(1-q)}\right]
    %     &> \frac{\lambda}{\ln 2}\left[\frac{q - \frac{1}{\frac{1}{q} + \frac{1}{\lambda - 1}\frac{1 - q}{q}\ln \frac{1 - q}{q}}}{q(1-q)}\right]\\
    %     &= \frac{\lambda}{\lambda + (1-q)\ln \frac{1-q}{q}-1}\log \frac{1-q}{q}\\
    %     &>\log \frac{1-q}{q}\\
    %     &= H'(q).
    % \end{split}   
    % \]
    When $L^* \geq 0.45$, for any $\xi \in (H(L^*), \frac{1+H(L^*)}{2})$ and any $\lambda > 1$, we have $U_{\lambda}^{-1}(\xi) \geq U_{\lambda}^{-1}(H(L^*)) \geq H^{-1}(H(L^*)) = L^* \geq 0.45$ due to monotonicity of $U_{\lambda}$. Note that for $0.217 < q < 0.5$, we have $(1-q)\ln \frac{1 - q}{q} < 1$, and thus $U'_{\lambda}(q)>H'(q)$  by Taylor expansion of $p^*$. Therefore, $\forall \lambda > 1$, $\forall L^* \geq 0.45$, $\forall \xi \in (H(L^*), \frac{1+H(L^*)}{2})$, we have
    \begin{equation}
        U'_{\lambda}\left(U_{\lambda}^{-1}(\xi)\right) > H'\left(U_{\lambda}^{-1}(\xi)\right) > H'\left(U_1^{-1}(\xi)\right)>H'\left(U_1^{-1}\left(\frac{1+H(L^*)}{2}\right)\right)= H'\left(1 - 2^{-\frac{1+H(L^*)}{2}}\right).
    \end{equation}
    Combining case (1) and (2) yields that $\forall \lambda > 1$, $\forall L^* \in (0,0.5)$, $\forall \xi \in (H(L^*), \frac{1+H(L^*)}{2})$, $U'_{\lambda}\left(U_{\lambda}^{-1}(\xi)\right) \geq \min\left(c, H'\left(1 - 2^{-\frac{1+H(L^*)}{2}}\right)\right)$. This proves the first half of \autoref{lemmaA.1}. By Taylor expansion, we have $H'\left(1 - 2^{-\frac{1+H(L^*)}{2}}\right) > \frac{2}{\ln 2}(L^* - \frac{1}{2})^2$. This yields that $\frac{1}{U_{\lambda}'(U_{\lambda}^{-1}(\xi))} \leq \frac{1}{\min\left(c, H'\left(1 - 2^{-\frac{1+H(L^*)}{2}}\right)\right)} = O\left(\frac{1}{(L^* - \frac{1}{2})^2}\right)$. 

    % Denote $f(L^*) = 1 - 2^{-\frac{1+H(L^*)}{2}}$ and $g(L^*) = H'(f(L^*))$. By Taylor expansion of $g$ around $\frac{1}{2}$, we get $g(L^*) = g(\frac{1}{2}) + g'(\frac{1}{2})(L^* - \frac{1}{2}) + \frac{1}{2}g''(\frac{1}{2})(L^* - \frac{1}{2})^2 + \frac{1}{6}g^{(3)}(\eta)(L^* - \frac{1}{2})^3$ for some $\eta \in (L^*, \frac{1}{2})$. Note that $g(\frac{1}{2}) = g'(\frac{1}{2}) = 0$, and $g''(\frac{1}{2}) = \frac{4}{\ln 2}$. And by chain rule, $g^{(3)}(\eta) = H^{(4)}(f(\eta))f'(\eta) + 2H^{(3)}(f(\eta))f''(\eta) + H''(f(\eta))f^{(3)}(\eta)$, each term of which is negative for $\eta < 0.5$, and so $g^{(3)}(\eta)< 0$. Hence $g(L^*) > \frac{2}{\ln 2}(L^* - \frac{1}{2})^2$. This yields that $\frac{1}{U_{\lambda}'(U_{\lambda}^{-1}(\xi))} \leq \frac{1}{\min\left(c, g(L^*)\right)} = O(\frac{1}{(L^* - \frac{1}{2})^2})$.

    This completes the proof of \autoref{lemmaA.1}.
\end{proof}
Combining \autoref{lemmaA.1}, \eqref{eqn_15}, and \eqref{eqn__20}, for $m \gtrsim \max\{\frac{2\lambda}{1 - H(L(h^*))}\log^2(\frac{2\lambda}{1 - H(L(h^*))}), \frac{4}{(1 - H(L(h^*)))^2}\log^3(\frac{4}{(1 - H(L(h^*)))^2})\}$ such that $\Delta < \frac{1}{2}(1 - H(L(h^*)))$, we have 
\begin{align}
    \mathbb{E}L(\MDL_{\lambda}(S)) &\leq U_{\lambda}^{-1}\bigg(H(L(h^*)) + \Delta \bigg) + \frac{1}{\sqrt{m}}\\
    &= U_{\lambda}^{-1}( H(L(h^*)) ) + \frac{1}{U_{\lambda}'(U_{\lambda}^{-1}(\xi))}\Delta + \frac{1}{\sqrt{m}}\\
    &\leq U_{\lambda}^{-1}( H(L(h^*)) ) + \frac{1}{\min\left(c, \frac{2}{\ln 2}(L(h^*) - \frac{1}{2})^2\right)}\Delta + \frac{1}{\sqrt{m}} \label{eqn3.1rhs}\\
    &= U_{\lambda}^{-1}( H(L(h^*)) ) + O\left(\frac{1}{(1 - 2L(h^*))^2} \cdot \left(\lambda\left(\frac{\abs{h^*}_{\pi} + \log m}{m}\right) + \sqrt{\frac{\log^3 (m)}{m}}\right)\right).\label{eq58}
\end{align}

On the other hand, if $m$ is small such that $\Delta \geq \frac{1}{2}(1 - H(L(h^*)))$, then by Taylor expansion, $1 - H(L(h^*)) > \frac{2}{\ln 2}(L(h^*) - \frac{1}{2})^2$. But then the right hand side of \eqref{eqn3.1rhs} $ \geq \frac{1}{\frac{2}{\ln 2}(L(h^*) - \frac{1}{2})^2}\Delta \geq \frac{1}{\frac{2}{\ln 2}(L(h^*) - \frac{1}{2})^2}\cdot \frac{1}{2}(1 - H(L(h^*))) > \frac{1}{\frac{2}{\ln 2}(L(h^*) - \frac{1}{2})^2}\cdot \frac{1}{2} \cdot \frac{2}{\ln 2}(L(h^*) - \frac{1}{2})^2 = \frac{1}{2}$. As a result, the bound is vacuously true. 

Therefore, the bound \eqref{eq58} holds for any $m$. This completes the proof of \autoref{MDL UB}.
\end{proof}
Next, we prove the finite sample guarantee of \autoref{lambda infty}, and then use it to derive the consistency result when $\lambda \rightarrow \infty$ presented in \autoref{lambda_infty_coro}. 

\subsection{Proof of \autoref{lambda infty} and \autoref{lambda_infty_coro}}\label{AppendixA.2}

\restate{Theorem}{lambda infty} 
\begin{restatethm} 
For any predictor $h^*$, source distribution $D$, valid prior $\pi$, and any $\lambda > 1$ and $m$:
\begin{equation}\label{thm 3.4 bd}
    \underset{S \sim D^m}{\mathbb{E}}[L(\MDL_{\lambda}(S))] 
    \leq L(h^*)+ O\left(\frac{1}{1 - 2L(h^*)} \cdot \left(\frac{1}{\lambda} + \lambda\left(\frac{\abs{h^*}_{\pi} + \log m}{m}\right) + \sqrt{\frac{\log^3 (m)}{m}}\right)
    \right),
\end{equation}
where $O(\cdot)$ only hides an absolute constant, that does not depend on $D, \pi$ or anything else.
\end{restatethm}
\begin{proof}
we first prove \autoref{lemma9} which quantifies how close the binary entropy function $H(q)$ is to the function $U_{\lambda}(q) = \lambda \cdot \KL(\frac{1}{1+ (\frac{1-q}{q})^{\frac{\lambda}{\lambda - 1}}}\Vert q) + H(\frac{1}{1+ (\frac{1-q}{q})^{\frac{\lambda}{\lambda - 1}}})$ for $\lambda > 1$.

\restate{Lemma}{lemma9} 
\begin{restatethm}
For any $\lambda >1$ and any $0 \leq q \leq \frac{1}{2}$, $H(q) < U_{\lambda}(q)+ \lambda/(\lambda-1)^2$.
\end{restatethm}
\begin{proof}{\bf of \autoref{lemma9}:} 
    Letting $p^* = \frac{1}{1+ (\frac{1-q}{q})^{\frac{\lambda}{\lambda - 1}}}$, and $U_{\lambda}(q) = \lambda \cdot \KL(p^*\Vert q) + H(p^*)$ and
    \begin{equation}\label{phipstar}
        \log \frac{p^*}{1-p^*} = \frac{\lambda}{\lambda - 1}\log \frac{q}{1-q}.
    \end{equation}
Note that for $\lambda > 1$, we have $0 < p^* < q \leq \frac{1}{2}$. Denote $\delta = q - p^* > 0$, and denote the function $\phi (q) = \log \frac{q}{1-q}$. Then the relationship between $p^*$ and $q$ in \eqref{phipstar} can be rewritten as 
\begin{equation}\label{phipstarnew}
     \phi(p^*) = \frac{\lambda}{\lambda - 1}\phi(q).    
\end{equation}
Note that $\phi (q) < 0$ for $q \in [0, \frac{1}{2}]$, and its first-order derivative $\phi' (q) = \frac{1}{q(1-q)} > 0$ is positive and monotonically decreasing on $[0, \frac{1}{2}]$. Hence, by the mean value theorem and monotonicity of the derivative of $\phi$, we have
\begin{equation}\label{delta_mvt}
\begin{split}
  \phi(q) - \phi(p^*) &= \phi'(\xi_0)\cdot\delta \text{ , for some } \xi_0 \in (p^*, q)\\
  &\geq \phi'(q)\cdot\delta
  = \frac{1}{q(1-q)}\cdot\delta.
\end{split}
\end{equation}
Plugging \eqref{phipstarnew} into \eqref{delta_mvt}, we get an upper bound for $\delta$ in terms of $q$ such that 
\begin{equation}\label{deltaub}
    \delta \leq -\frac{q(1-q)}{\lambda - 1}\phi(q).
\end{equation}
% Now, note that $H(q) - U_{\lambda}(q) = \left(H(q) - H(p^*)\right) - \lambda \KL(p^*\Vert q) \leq H(q) - H(p^*)$. So to find an upper bound for $H(q) - U_{\lambda}(q)$, it suffices to upper bound $H(q) - H(p^*)$.
Note that $H'(q) = \log \frac{1 - q}{q} \geq 0$ is positive and monotonically decreasing on $[0, \frac{1}{2}]$, so by mean value theorem,
\begin{equation}\label{H_mvt}
    \begin{split}
        H(q) - H(p^*) &= H'(\xi_1)\cdot\delta\text{ , for some } \xi_1 \in (p^*, q)\\
        &\leq H'(p^*)\cdot\delta
        = \log \frac{1 - p^*}{p^*}\cdot\delta
        = \frac{\lambda}{\lambda - 1}\log \frac{1-q}{q} \cdot \delta,
    \end{split}
\end{equation}
where the last equality follows from \eqref{phipstarnew}.

By plugging the upper bound \eqref{deltaub} for $\delta$ into \eqref{H_mvt}, we get an upper bound for $H(q) - H(p^*)$, and thus also an upper bound for $H(q) - U_{\lambda}(q)$:
\[
\begin{split}
    H(q) - U_{\lambda}(q) &= H(q) - H(p^*) - \lambda \KL(p^*\Vert q)
    \leq H(q) - H(p^*)\\
    &\leq \frac{\lambda}{\lambda - 1}\log \frac{1-q}{q} \cdot \left(-\frac{q(1-q)}{\lambda - 1}\phi(q)\right)
    = \frac{\lambda}{(\lambda - 1)^2}q(1-q)\phi^2(q).
\end{split}
\]
Using the fact that $q(1-q)\phi^2(q) < 1$ for $q \in [0, \frac{1}{2}]$, we prove the desired result $H(q) - U_{\lambda}(q) < \frac{\lambda}{(\lambda - 1)^2}$. This completes the proof of \autoref{lemma9}. 
\end{proof}
Combining \autoref{lemma9} with inequality \eqref{lambda>1hpb}, we have with probability $\geq 1 - \delta$, 
\begin{equation}
    H(L(\MDL_{\lambda}(S))) \leq H(L(h^*)) + C\sqrt{\frac{2(\log m)^2 \cdot \log \frac{1}{\delta/2}}{m}}
    + {\lambda}\frac{\log(\frac{m+1}{\delta/2})}{m}+{\lambda}\frac{\abs{h^*}_{\pi}}{m} + \frac{\lambda}{(\lambda-1)^2}.
\end{equation}
Taking $\delta = \frac{1}{\sqrt{m}}$ yields that with probability $\geq 1 - \frac{1}{\sqrt{m}}$,  for some constant $C', C''$,
\begin{equation}\label{eqn30}
    H(L(\MDL_{\lambda}(S))) \leq H(L(h^*)) + C'\frac{(\log m)^\frac{3}{2}}{\sqrt{m}}
    + {\lambda}C''\frac{\log m}{m}+{\lambda}\frac{\abs{h^*}_{\pi}}{m} + \frac{\lambda}{(\lambda-1)^2}.
\end{equation}
Let $\Delta = C'\frac{(\log m)^\frac{3}{2}}{\sqrt{m}}
    + {\lambda}C''\frac{\log m}{m}+\lambda \frac{\abs{h^*}_{\pi}}{m} + \frac{\lambda}{(\lambda-1)^2}$. If the right hand side of \eqref{eqn30} $H(L(h^*)) + \Delta \leq 1$, then it is well-defined to take the inverse function $H^{-1}$ on both sides of \eqref{eqn30} to yield that with probability $\geq 1 - \frac{1}{\sqrt{m}}$, 
\begin{align}\label{eqn33}
    L(\MDL_{\lambda}(S)) &\leq H^{-1}\left(H(L(h^*))+ \Delta \right).
\end{align}
By the mean value theorem, we have 
\begin{equation}\label{eqnmvt34}
\begin{split}
    H^{-1}( H(L(h^*)) + \Delta ) &=H^{-1}( H(L(h^*)) ) + (H^{-1})'(\xi)\Delta \\
    &= L(h^*) + (H^{-1})'(\xi)\Delta,
\end{split}   
\end{equation}
for some $\xi \in (H(L(h^*)), H(L(h^*)) + \Delta)$.

Because the entropy function $H$ is concave, the inverse function $H^{-1}$ is convex on $(0, \frac{1}{2})$. By the convexity of $H^{-1}$, the derivative $(H^{-1})'(\xi)$ is always upper bounded by the slope of the line interpolating the two points $(H(L(h^*)), L(h^*))$ and $(1, \frac{1}{2})$, i.e.
\begin{equation}\label{eqn36}
    (H^{-1})'(\xi) \leq \frac{\frac{1}{2} - L(h^*)}{1 - H(L(h^*))}.
\end{equation}

Combining \eqref{eqn33}, \eqref{eqnmvt34}, and \eqref{eqn36}, we get with probability $\geq 1 - \frac{1}{\sqrt{m}}$,
\begin{equation}\label{thm3.4hpb}
    L(\MDL_{\lambda}(S)) \leq L(h^*) + \frac{\frac{1}{2} - L(h^*)}{1 - H(L(h^*))}\Delta.
\end{equation}

Note that although we assume $H(L(h^*)) + \Delta \leq 1$ and take $H^{-1}$ inverse function to get \eqref{thm3.4hpb}, when $H(L(h^*)) + \Delta > 1$, the bound \eqref{thm3.4hpb} is vacuously true. Indeed, if $H(L(h^*)) + \Delta > 1$, then $\Delta > 1 - H(L(h^*))$, and thus 
$\frac{\frac{1}{2} - L(h^*)}{1 - H(L(h^*))}\Delta > \frac{1}{2} - L(h^*)$ and the right hand side of \eqref{thm3.4hpb} $> \frac{1}{2}$. In this case, \eqref{thm3.4hpb} vacuously holds. Hence, \eqref{thm3.4hpb} holds for any $L(h^*)$ and $\Delta$.

By Taylor expansion of $H(L(h^*))$ around $\frac{1}{2}$, we get $H(L(h^*)) = H(\frac{1}{2}) + H'(\frac{1}{2})(L(h^*) - \frac{1}{2}) + \frac{H''(\frac{1}{2})}{2}(L(h^*) - \frac{1}{2})^2 + \frac{H'''(\xi)}{6}(L(h^*) - \frac{1}{2})^3$ for some $\xi \in (L(h^*), \frac{1}{2})$. Since $H(\frac{1}{2}) = H'(\frac{1}{2}) = 0$, $ H''(\frac{1}{2}) = -\frac{4}{\ln 2}$, and $H'''(\xi) > 0, \forall \xi < \frac{1}{2}$, this gives us $1 - H(L(h^*)) \geq \frac{2}{\ln}(L(h^*) - \frac{1}{2})^2$. Hence, $\frac{\frac{1}{2} - L(h^*)}{1 - H(L(h^*))} \leq \frac{\ln 2}{1 - 2L(h^*)}$.

Since the risk is bounded, the high probability bound \eqref{thm3.4hpb} implies the bound on expected risk:
\begin{align}
    \underset{S \sim \mathcal{D}^m}{\mathbb{E}}[L(\MDL_{\lambda}(S))] 
    &\leq L(h^*)+ \frac{\frac{1}{2} - L(h^*)}{1 - H(L(h^*))}\Delta  + \frac{1}{\sqrt{m}}\\
    &\leq L(h^*)+ O\left(\frac{1}{1 - 2L(h^*)} \cdot  \left(\frac{1}{\lambda} + \lambda\left(\frac{\abs{h^*}_{\pi} + \log m}{m}\right) + \sqrt{\frac{\log^3 (m)}{m}}\right)
    \right).
\end{align}
This completes the proof of \autoref{lambda infty}.
\end{proof}
We can then use \autoref{lambda infty} to derive the consistency result as shown in the \autoref{lambda_infty_coro}:
\restate{Corollary}{lambda_infty_coro} 
\begin{restatethm}
For $1 \ll \lambda_m \ll m/\log m$ and any $h^*$ with $\pi(h^*)>0$, $\underset{m \rightarrow \infty}{\lim} \underset{\pi, D }{\sup} \underset{S \sim D^m}{\mathbb{E}}[L(\MDL_{\lambda_m})] \leq L(h^*) = L^*$.
\end{restatethm}
\begin{proof}
    Since $1 \ll \lambda_m \ll m/\log m$ , as $m \rightarrow \infty$, all the terms inside the big-O notation of the right hand side of \ref{thm 3.4 bd} in \autoref{lambda infty} vanish, yielding the consistency result.
\end{proof}
\section{Lower Bound Constructions and Proofs}\label{AppendixB}
In this Section, we provide the detailed lower bound proofs for \autoref{MDL LB}, \autoref{lambda0} and \autoref{lambda>>m}.
\subsection{Lower Bound for $0 < \lambda < \infty$ (proof of \autoref{MDL LB})}\label{AppendixB.1}
\restate{Theorem}{MDL LB} 
\begin{restatethm}[Agnostic Lower Bound]
    For any $0<\lambda<\infty$, any $L^* \in (0,0.5)$ and $L^* \leq L' <\ell_{\lambda}(L^*)$, there exists a prior $\pi$, a hypothesis $h^*$ with $\pi(h^*) \geq 0.1$ and source distribution $D$ with $L_D(h^*) = L^*$ such that $\mathbb{E}_S \left[L_D(\MDL_{\lambda}(S))\right] \rightarrow L'$ as sample size $m \rightarrow \infty$. 
\end{restatethm}
Consider the source distribution and prior described in \autoref{sec:finite LB proof}. We prove that with probability one, as $m\rightarrow\infty$, $\MDL_\lambda$ will select one of the ``bad'' predictors, i.e. there exists some $i \geq 1$ such that $J_\lambda(h_i,S) < J_\lambda(h_0,S)$ and $L_S(h_i) < \frac{1}{2}$, i.e.
    \begin{equation}\label{eqn_lb_defn}
        \lambda |h_i|_{\pi} + \log {m \choose m L_S(h_i)} < \lambda |h_0|_{\pi} + \log {m \choose m L_S(h_0)}.
    \end{equation}
    This is equivalent to analyzing its approximation $\tilde{J}_\lambda(h,S)=\lambda \abs{h}_\pi + m H(L_S(h))$ (see equation \eqref{stirling}), and by rearranging and dividing by $m$ on both sides
    \begin{equation} \label{1}
        \frac{\lambda|h_i|_{\pi}}{m} + (H(L_S(h_i)) - H(L_S(h_0)))\leq \frac{\lambda \log 10 - \log (m+1)}{m}.
    \end{equation}
    Notice that the right hand side of \eqref{1} is deterministic and converges to zero as $m \rightarrow \infty$. Thus, to show \eqref{1}, it suffices to show that there exists $i>0$ such that as $m \rightarrow \infty$, the left hand side of \eqref{1} is negative with probability one. And the proof is different for $\lambda \leq 1$ and $\lambda > 1$, as we will discuss separately below. 

    We first prove that $H(L_S(h_0)) \rightarrow H(L^*)$ almost surely, which will be repeatedly used in the proofs.
    
    \begin{lemma}\label{lemma 10.1}
        % With probability $\geq 1-e^{-2 \sqrt{m}}$, we have $H(L_S(h_0)) \geq  H(L^*) - \frac{\log (m-1)}{m^{1/4}}$.
        $H(L_S(h_0))$ converges to $H(L^*)$ almost surely.
    \end{lemma}
    % \begin{proof}
    %     By the Chernoff bound, we have $\mathbb{P}(L_S(h_0)<L_D(h_0) - m^{-\frac{1}{4}}) \leq e^{-2\sqrt{m}}$. Hence, with probability $\geq 1 - e^{-2\sqrt{m}}$, $L_S(h_0) \geq L_D(h_0) - m^{-\frac{1}{4}} = L^* - m^{-\frac{1}{4}}$.

    %     If $L^* < L_S(h_0)$, then by the monotonicity of binary entropy, $H(L^*) < H(L_S(h_0))$ and we are done.

    %     Since $L_S(h_0)$ can only take values in $\{ \frac{1}{m}, \frac{2}{m}, \cdots\}$, we are only left with the case $L^* \geq L_S(h_0) \geq \frac{1}{m}$. But then by the concavity of binary entropy, $H(L_S(h_0)) \geq H(L^*) - H'(\frac{1}{m})\cdot (L_S(h_0) - L^*) \geq H(L^*) - H'(\frac{1}{m})\cdot m^{-\frac{1}{4}} = H(L^*) - \frac{\log (m-1)}{m^{\frac{1}{4}}}$, where we use the fact that the derivative of $H$ at a point $x$ is $H'(x) = \log \frac{1-x}{x}$.

    %     Combining the two cases, we prove the result.
    % \end{proof}
    \begin{proof}{\bf of \autoref{lemma 10.1}:}
        % This is equivalent to showing for any fixed $\epsilon > 0$, $\mathbb{P}\left(\left|H(L_S(h_0)) - H(L^*)\right| > \epsilon \text{ i.o.} \right) = 0$, where `i.o.' stands for infinitely often. By the Borel-Cantelli Lemma, it suffices to show that $\sum_{m = 1}^\infty \mathbb{P}\left(\left|H(L_S(h_0)) - H(L^*)\right| > \epsilon \right) < \infty$.  
        For fixed $\epsilon > 0$, there exists an $M > 0$ such that $\left\{\left|L_S(h_0) - L^*\right| > \epsilon\right\} \subseteq \left\{\left|L_S(h_0) - L^*\right| > m^{-\frac{1}{4}}\right\}$, for all $m > M$. This implies that for all $m > M$, $\mathbb{P}\left(\left|L_S(h_0) - L^*\right| > \epsilon \right) \leq \mathbb{P}\left(\left|L_S(h_0) - L^*\right| >  m^{-\frac{1}{4}}\right) \leq 2e^{-2\sqrt{m}}$, where the second inequality is by Chernoff bound. Therefore, $\sum_{m = 1}^\infty \mathbb{P}\left(\left|L_S(h_0) - L^*\right| > \epsilon \right) \leq \sum_{m = 1}^M 1 + \sum_{m > M}2e^{-2\sqrt{m}} < \infty$. By Borel-Cantelli Lemma, this proves $\mathbb{P}\left(L_S(h_0) \rightarrow L^* \text{, as } m \rightarrow \infty\right)= 1$, which implies that $\mathbb{P}\left(H(L_S(h_0)) \rightarrow H(L^*) \text{, as } m \rightarrow \infty\right)= 1$ since $H$ is continuous.
    \end{proof}
    Now we give a proof of inequality \eqref{1} based on $\lambda$ values: $0 < \lambda \leq 1$ and $\lambda > 1$. 
\subsubsection{$0 < \lambda \leq 1$}\label{section 6.1.2}
\begin{proof}
    We first prove the following claim:\\
    \textbf{Claim:} for some function $k(m) = \frac{2\sqrt{m}}{(1-L')^m}$, with probability one, there exists some `bad' classifier $h_{\hat{i}}$ with $0< \hat{i} \leq k(m)$ such that $L_S(h_{\hat{i}}) = 0$ for all but finitely many $m$.
    
    % with probability larger than $1-e^{-2 \sqrt{m}}$, there exists some `bad' classifier $h_{\hat{i}}$ with $0< \hat{i} \leq k(m)$ such that $L_S(h_{\hat{i}}) = 0$.
    \begin{proof} \textbf{of the claim:}
        Let $k$ be a positive integer and $\mathcal{H}_k = \{h_j \in \mathcal{H}: 1 \leq j \leq k\}$. Then we have $\mathbb{P}(\forall h \in \mathcal{H}_k, L_S(h) > 0) = (1-(1-L')^m)^k \leq e^{-k(1-L')^m}$ , which the first equality follows from independence and the inequality by $\forall x \in [0,1], k > 0: (1-x)^k \leq e^{-kx}$. Now we set $k = k(m) = \frac{2\sqrt{m}}{(1-L')^m}$. Plugging in, we get $\mathbb{P}(\forall h \in \mathcal{H}_k, L_S(h) > 0) \leq e^{-2 \sqrt{m}}$. So $\sum_{m = 1}^\infty \mathbb{P}(\forall h \in \mathcal{H}_k, L_S(h) > 0) < \infty$. Consequently, by Borel-Cantelli, $\mathbb{P}(\exists h_{\hat{i}} \text{ with } 0< \hat{i} \leq k(m) \text{ s.t. }L_S(h_{\hat{i}}) = 0 \text{ for all but finitely many }m) = 1$.
    \end{proof}
    
   By the definition of $\pi$ and that $h_{\hat{i}} \in \mathcal{H}_{k(m)}$, we have $|h_{\hat{i}}|_{\pi} \leq m\log \frac{1}{1-L'} + C\log m$, for some constant $C>0$.
   
    By \autoref{lemma 10.1} and the claim, with probability one, the limit of the left hand side of \eqref{1} satisfies
    \begin{equation}
    \begin{split}
         \lim_{m \rightarrow \infty} \lambda\frac{|h_{\hat{i}}|_{\pi}}{m} + H(L_S(h_{\hat{i}})) - H(L_S(h_0)) 
         &\leq \lim_{m \rightarrow \infty} \lambda \log\frac{1}{1-L'}+ H(L_S(h_{\hat{i}})) - H(L_S(h_0))+ C\lambda\frac{\log m}{m} \\
         &= \lambda \log\frac{1}{1-L'} - H(L^*)
    \end{split}
    \end{equation}
    Thus, as long as $L' < 1 - 2^{-H(L^*)/{\lambda}}$, as $m \rightarrow \infty$, the limit of the left hand side of \eqref{1} is negative with probability one. This does not mean $\MDL_{\lambda}$ necessarily outputs $h_{\hat{i}}$, but this implies that $\MDL_{\lambda}$ will output some $h_{i}$ with $i > 0$, and hence $L_D(\MDL_{\lambda}(S)) = L'$ with probability one, which implies the bound for the expected risk:
    as $m \rightarrow \infty$, $\mathbb{E} L_D(\MDL_{\lambda}(S)) \rightarrow L'$.  This completes the proof for $ 0< \lambda \leq 1$.
\end{proof}  
\subsubsection{$1 < \lambda < \infty$}
\begin{proof}
    We first prove the following claim:\\
    \textbf{Claim:} for some function $k(m) = 2^{m\KL(\hat{L}\Vert L')}$ where $\hat{L} = \frac{1}{1+ (\frac{1-L'}{L'})^\frac{\lambda}{\lambda - 1}}$, let $h_{\hat{i}}$ be the predictor that achieves the smallest empirical error among $\mathcal{H}_k = \{h_j \in \mathcal{H}: 1 \leq j \leq k(m)\}$, i.e. $L_S(h_{\hat{i}}) = \min_{1 \leq i \leq k(m)} L_S(h_i)$. Then we have $H(L_S(h_{\hat{i}}))$ converges to $H(\hat{L})$ almost surely.

    \begin{proof}\textbf{of the claim:}
        Note that $L_S(h_{\hat{i}})$ is a minimum of i.i.d Binomial random variables. Denote $\Delta \coloneqq 2\log \sqrt{2}m + 4\log (m+1) + \left[ \log \frac{L'}{1-L'} \right]_+$. There exists an $M_1 > 0$ such that for all $m > M_1$, we have $\frac{\Delta}{m} < \KL(\hat{L}\Vert L')$. Then by the KL bound of the minimum of i.i.d Binomials ( \autoref{thm12} in  \autoref{AppendixC}), for all $m > M_1$, we have with probability $1-\frac{1}{m^2}$, 
        \begin{equation}
        \begin{split}\label{KL bd}
            \KL(L_S(h_{\hat{i}})\Vert L') &= \frac{\log k(m) \pm \Delta}{m}
            = \KL(\hat{L}\Vert L') \pm \frac{\Delta}{m}
        \end{split}
        \end{equation}
        \begin{equation}\label{Z<p}
            \textrm{and}\quad\quad L_S(h_{\hat{i}}) < L'.
        \end{equation}

        We first show that the KL bound \eqref{KL bd} implies that $\KL(L_S(h_{\hat{i}})\Vert L')$ converges to $ \KL(\hat{L}\Vert L')$ almost surely, i.e., $\KL(L_S(h_{\hat{i}})\Vert L') \rightarrow \KL(\hat{L}\Vert L')$ as $m \rightarrow \infty$, with probability one. This is equivalent to showing $\mathbb{P}\left(\left|\KL(L_S(h_{\hat{i}})\Vert L') - \KL(\hat{L}\Vert L')\right| > \epsilon \text{ i.o.} \right) = 0$, for any fixed $\epsilon > 0$, where `i.o.' stands for infinitely often. By the Borel-Cantelli Lemma, it suffices to show that $\sum_{m = 1}^\infty \mathbb{P}\left(\left|\KL(L_S(h_{\hat{i}})\Vert L') - \KL(\hat{L}\Vert L')\right| > \epsilon \right) < \infty$. 

        Note that for fixed $\epsilon > 0$, there exists an $M_2 > 0$ such that $\left\{\left|\KL(L_S(h_{\hat{i}})\Vert L') - \KL(\hat{L}\Vert L')\right| > \epsilon\right\} \subseteq \left\{\left|\KL(L_S(h_{\hat{i}})\Vert L') - \KL(\hat{L}\Vert L')\right| > \frac{\Delta}{m}\right\}$, for all $m > M_2$. This implies that for all $m > M \coloneqq \max(M_1, M_2)$, $\mathbb{P}\left(\left|\KL(L_S(h_{\hat{i}})\Vert L') - \KL(\hat{L}\Vert L')\right| > \epsilon \right) \leq \mathbb{P}\left(\left|\KL(L_S(h_{\hat{i}})\Vert L') - \KL(\hat{L}\Vert L')\right| > \frac{\Delta}{m}\right) \leq \frac{1}{m^2}$, where the second inequality follows from \eqref{KL bd}.
        Therefore, $\sum_{m = 1}^\infty \mathbb{P}\left(\left|\KL(L_S(h_{\hat{i}})\Vert L') - \KL(\hat{L}\Vert L')\right| > \epsilon \right) \leq \sum_{m = 1}^M 1 + \sum_{m > M}\frac{1}{m^2} < \infty$. By the Borel-Cantelli Lemma, this proves
        \begin{equation}\label{KLa.s.}
            \mathbb{P}\left( \KL(L_S(h_{\hat{i}})\Vert L') \rightarrow \KL(\hat{L}\Vert L') \text{ as } m \rightarrow \infty\right)= 1 
        \end{equation}
        By the same argument, since $\sum_{m = 1}^\infty \mathbb{P}\left( L_S(h_{\hat{i}}) \geq L' \right) \leq \sum_{m = 1}^{M_1}1 + \sum_{m = M_1 + 1}^\infty \frac{1}{m^2} < \infty$ given by \eqref{Z<p}, by Borel-Cantelli, we have $\mathbb{P}(L_S(h_{\hat{i}}) \geq L' \text{ i.o.}) = 0$. This shows that  
        \begin{equation}\label{Z<pa.s}
            \mathbb{P} \left(L_S(h_{\hat{i}}) < L' \text{ for all but finitely many } m \right) = 1
        \end{equation}
        Then by the continuity of KL and the fact that for any $p,q<r$, $\KL(p||r)=\KL(q||r)$ if and only if $p=q$, equation \eqref{KLa.s.} and \eqref{Z<pa.s} implies that $L_S(h_{\hat{i}}) \rightarrow \hat{L}$ almost surely, which implies that $H(L_S(h_{\hat{i}})) \rightarrow H(\hat{L})$ almost surely since $H$ is continuous. This proves the claim.
        % Likewise, we can also prove $H(L_S(h_0)) \rightarrow H(L^*)$ a.s. using the same arguments: By Chernoff bound, we have with probability $1 - 2e^{-2\sqrt{m}}$, $\left|L_S(h_0) - L^* \right| < m^{-\frac{1}{4}}$. Then by applying Borel-Cantelli, we can prove $L_S(h_0) \rightarrow L^*$ a.s, which implies $H(L_S(h_0)) \rightarrow H(L^*)$ a.s.
    \end{proof}   
    By the definition of $\pi$ and that $h_{\hat{i}} \in \mathcal{H}_{k(m)}$, we have $|h_{\hat{i}}|_{\pi} \leq m\KL(\hat{L}\Vert L') + C'\log m$, for some $C' > 0$.
    
   By \autoref{lemma 10.1} and the claim, with probability one, the limit of the left hand side of \eqref{1} satisfies
    \begin{equation}\label{20}
    \begin{split}
         \lim_{m \rightarrow \infty} \lambda\frac{|h_{\hat{i}}|_{\pi}}{m} + H(L_S(h_{\hat{i}})) - H(L_S(h_0)) 
         &\leq  \lim_{m \rightarrow \infty}\lambda\KL(\hat{L}\Vert L') +  H(L_S(h_{\hat{i}})) - H(L_S(h_0)) + C'\lambda\frac{\log m}{m} \\
         &= \lambda\KL(\hat{L}\Vert L') + H(\hat{L}) - H(L^*) 
         = U_{\lambda}(L') - H(L^*),
    \end{split}
    \end{equation}
    where $U_{\lambda}(L') = \lambda \KL(\hat{L} \Vert L') + H(\hat{L})$.

    Hence, as long as  $L' < U_{\lambda}^{-1} ( H(L^*) )$, as $m \rightarrow \infty$, the left hand side of \eqref{1} is negative with probability one. It is important to note that in the definition of $\MDL_{\lambda}$, we also require the selected hypothesis $h$ to satisfy $L_S(h) \leq
    \frac{1}{2}$ (as in equation \eqref{defn:MDL}). And we just showed that with probability one $L_S(h_{\hat{i}}) \rightarrow \hat{L} < L' < U_{\lambda}^{-1} ( H(L^*) ) < \frac{1}{2}$, so $h_{\hat{i}}$ satisfies the condition and has a lower $\MDL$ objective than $h_0$.
    
    This implies that $\MDL_{\lambda}$ will output some $h_{i}$ with $i >0$, and hence $L_D(\MDL_{\lambda}(S)) = L'$ with probability one, which implies the bound for the expected risk:
    as $m \rightarrow \infty$, $\mathbb{E} L_D(\MDL_{\lambda}(S)) \rightarrow L'$.  This completes the proof for $ 1< \lambda < \infty$.
\end{proof}
This completes the proof for \autoref{MDL LB}.
\subsection{Proof of \autoref{lambda0}}\label{AppendixB.2}
\restate{Theorem}{lambda0} 
\begin{restatethm}
For any $\lambda_m \rightarrow 0$ or $\lambda = 0$, any $L^* \in (0,0.5)$, and $L^* \leq L' < 1$, there exists a prior $\pi$, a hypothesis $h^*$ with $\pi(h^*) \geq 0.1$ and source distribution $D$ with $L_D(h^*) = L^*$ such that $\mathbb{E}_S \left[L_D(\MDL_{\lambda_m}(S))\right] \rightarrow L'$ as sample size $m \rightarrow \infty$.
\end{restatethm}
\begin{proof}
Consider the same source distribution described in \autoref{sec:finite LB proof}, and with the same prior.  It is easy to see that the probability $h_0$ interpolates the data $S = \{(x_1, y_1), \cdots, (x_m, y_m)\}$ goes to 0 as the sample size $m \rightarrow \infty$. Indeed, $\mathbb{P}(h_0(x_t) = y_t, \forall t \in \{1,\cdots, m\}) = \mathbb{P}(x_t[0] = y_t, \forall t \in \{1,\cdots, m\}) = (1 - L^*)^m \rightarrow 0$ as $m \rightarrow \infty$, by independence of data. On the other hand, with probability one, there exists some $i > 0$ such that $h_i$ interpolates the data. To see this, the probability of its complement event $\mathbb{P}(\forall i > 0, x_t[i] \neq y_t \textnormal{ for some } t \in \{1,\cdots, m\}) \leq \sum_{t = 1}^m \mathbb{P}(\forall i > 0, x_t[i] \neq y_t) = m \cdot (L')^\infty = 0$, as long as $L' < 1$. Hence, $\mathbb{P}(\exists i > 0 \textnormal{ such that } L_S(h_i) = 0) = 1$. 

(1) $\lambda = 0$: $\MDL_{\lambda}$ simply minimizes $L_S(h)$ and will then always output some interpolating predictor. Therefore, as long as $L' < 1$, with probability one, $\MDL_{\lambda}$ returns some interpolating predictor $h_i$ with $i \geq 1$ as $m \rightarrow \infty$, implying that $L(\MDL_{\lambda} (S)) = L'$. Thus, $\mathbb{E}L(\MDL_{\lambda} (S)) \rightarrow L'$, and this completes the proof for $\lambda = 0$. 

(2) $\lambda_m \rightarrow 0$ as $m \rightarrow \infty$: let $\hat{i}$ denote the smallest index $\hat{i}\geq 1$ such that $L_S(h_{\hat{i}}) = 0$. We show with probability one, $J_{\lambda_m}(h_{\hat{i}},S) < J_{\lambda_m}(h_0,S)$ as sample size increases, i.e,
    \begin{equation}\label{eqn_lambda_0}
        \lambda_m |h_{\hat{i}}|_{\pi} + \log {m \choose m L_S(h_{\hat{i}})} < \lambda_m |h_0|_{\pi} + \log {m \choose m L_S(h_0)}.
    \end{equation}
    We already saw that this is equivalent to
    \begin{equation} \label{eqn_36}
        \frac{\lambda_m |h_{\hat{i}}|_{\pi}}{m} + (H(L_S(h_{\hat{i}})) - H(L_S(h_0)))\leq \frac{\lambda_m \log 10 - \log (m+1)}{m}.
    \end{equation}
    Since $\lambda_m \rightarrow 0$ as $m \rightarrow \infty$, the right hand side converges to 0. So it suffices to show that as $m \rightarrow \infty$, the limit of the left hand side of \eqref{eqn_36} is negative with probability one. 
    
% Note that $\hat{i}$ follows the geometric distribution with expectation $\mathbb{E} [\hat{i}] = \frac{1}{\mathbb{P}(L_S(h_i) = 0)} = \frac{1}{(1-L')^m}$, and variance $\textnormal{Var} [\hat{i}] = \frac{1 - \mathbb{P}(L_S(h_i) = 0)}{(\mathbb{P}(L_S(h_i) = 0))^2} =  \frac{1}{(1-L')^{2m}} -  \frac{1}{(1-L')^m}$. Then by Chebyshev's inequality, $\mathbb{P}(\left| \hat{i} - \mathbb{E}[\hat{i}] \right| \geq a) \leq \frac{\textnormal{Var} [\hat{i}]}{a^2}$. Taking $a = m \cdot  \frac{1}{(1-L')^m}$ yields that with probability larger than $1 - \frac{1}{m^2}$, $\hat{i} \leq (m+1) \frac{1}{(1-L')^m}$ and thus by the definition of $\pi$, $|h_{\hat{i}}|_{\pi} \leq m\log \frac{1}{1-L'} + C\log m$, for some constant $C>0$.

We already saw that $\hat{i} \leq\frac{m+1}{(1-L')^m}$ with probability approaching one. Hence, by the definition of $\pi$, $|h_{\hat{i}}|_{\pi} \leq m\log \frac{1}{1-L'} + C\log m$, for some constant $C>0$.

Therefore, by \autoref{lemma 10.1} and Borel-Cantelli, with probability one, the limit of the left hand side of \eqref{eqn_36} satisfies
    \begin{align}
         \lim_{m \rightarrow \infty} \lambda_m \frac{|h_{\hat{i}}|_{\pi}}{m} + H(L_S(h_{\hat{i}})) - H(L_S(h_0)) 
         % &=\lim_{m \rightarrow \infty} \lambda_m \frac{|h_{\hat{i}}|_{\pi}}{m} - H(L_S(h_0))\\
         \leq  \lim_{m \rightarrow \infty}\lambda_m \log \frac{1}{1-L'}- H(L_S(h_0)) + C\lambda_m\frac{\log m}{m} 
         = - H(L^*),\notag
    \end{align}
since $\lambda_m \rightarrow 0$ as $m \rightarrow \infty$.

So as long as $L' < 1$, the limit of the left hand side of \eqref{eqn_36} is negative with probability one, which implies that $\mathbb{E} L_D(\MDL_{\lambda_m}(S)) \rightarrow L'$ as $m \rightarrow \infty$.  This completes the proof for $\lambda_m \rightarrow 0$.

This completes the proof for \autoref{lambda0}.
% Then assuming the two large probability events, which hold with probability $\geq 1-e^{-2 \sqrt{m}} - \frac{1}{m^2}$ by the union bound, we plug in all these quantities ($|h_{\hat{i}}|_{\pi}, H(L_S(h_{\hat{i}})), H(L_S(h_0))$) into \eqref{eqn_36}. We can see that since $\lambda_m \rightarrow 0$ as $m \rightarrow \infty$, the left hand side of \eqref{eqn_36} is negative for all large $m$. Thus, $\mathbb{P}(L_D(\MDL_{\lambda_m}(S)) = L') \geq 1-e^{-2 \sqrt{m}} - \frac{1}{m^2}$, for all large $m$. Since the risk is bounded, the high probability bound implies the bound for the expected risk:
%     $\mathbb{E} L_D(\MDL_{\lambda_m}(S)) \leq L' + e^{-2\sqrt{m}} + \frac{1}{m^2}$, and $\mathbb{E} L_D(\MDL_{\lambda_m}(S)) \geq L' \cdot (1-e^{-2 \sqrt{m}}-\frac{1}{m^2})$. Remember all of these are based on $L' < 1$. Therefore, as long as $L' < 1$,  as $m \rightarrow \infty$, $\mathbb{E} L_D(\MDL_{\lambda_m}(S)) \rightarrow L'$.  This completes the proof for $\lambda_m \rightarrow 0$.
\end{proof}
\subsection{Proof of \autoref{lambda>>m}}\label{AppendixB.3}
\restate{Theorem}{lambda>>m} 
\begin{restatethm}
For any $\lambda_m = \Omega(m)$ with $\liminf \frac{\lambda_m}{m}> 10$, any $0 \leq L^* <0.5$, and any $L^* \leq L' < 0.5$, there exists a prior $\pi$, a hypothesis $h^*$ with $\pi(h^*) \geq 0.1$ and source distribution $D$ with $L_D(h^*) = L^*$ such that $\mathbb{E}_S \left[L_D(\MDL_{\lambda_m}(S))\right] \rightarrow L'$ as sample size $m \rightarrow \infty$.
\end{restatethm}
\begin{proof}
    % Consider two predictors $\{h_0, h_1\}$ such that $h_0(x)=x$, $h_1(x)=1$.

    % Given $L^* \in (0, \frac{1}{2})$, consider a source distribution $D$ where $y = 0$, and $x = y \oplus \text{Ber}(L^*)$. Consider the prior $\pi(h_0)=0.1$ and $\pi(h_1)= 0.9$. It is easy to check that $L(h_0) = L^*$ and $L(h_1) = 1$.

    Consider the same source distribution described in \autoref{sec:finite LB proof} but with only two predictors $\{h_0, h_1\}$ with the prior $\pi(h_0)=0.1$ and $\pi(h_1)= 0.9$. 
    
     We prove that with probability one, $J_{\lambda_m}(h_1,S) < J_{\lambda_m}(h_0,S)$ as sample size increases, i.e.
    \begin{equation}\label{eqn_lambda_omega_m}
        \lambda_m |h_1|_{\pi} + \log {m \choose m L_S(h_1)} < \lambda_m |h_0|_{\pi} + \log {m \choose m L_S(h_0)}.
    \end{equation}
    We already saw that this is equivalent to
    \begin{equation} \label{eqn39}
        -\frac{\lambda_m \log 9}{m} + H(L_S(h_1)) - H(L_S(h_0))< \frac{-\log (m+1)}{m}.
    \end{equation}

    Since the right hand side of \eqref{eqn39} is deterministic and converges to zero as $m \rightarrow \infty$, it suffices to show that as $m \rightarrow \infty$, the limit of the left hand side of \eqref{eqn39} is negative with probability one. 
    
    By \autoref{lemma 10.1}, with probability one, the limit of the left hand side of \eqref{eqn_36} satisfies
    \begin{equation}
         \lim_{m \rightarrow \infty} -\frac{\lambda_m \log 9}{m} + H(L_S(h_1)) - H(L_S(h_0)) \leq -10\log 9 + H(L') - H(L^*) < 0,
    \end{equation}
    where we used $\liminf \frac{\lambda_m}{m}> 10$, and $H(L_S(h_1)) \rightarrow H(L')$ a.s. (by the same proof as \autoref{lemma 10.1}). Hence, with probability one, $J_{\lambda_m}(h_1,S) < J_{\lambda_m}(h_0,S)$ as sample size increases.

    It is important to note that in the definition of $\MDL_{\lambda}$, we also require the selected hypothesis $h$ to satisfy $L_S(h) \leq
    \frac{1}{2}$ (as in equation \eqref{defn:MDL}). And we just showed that with probability one $L_S(h_1) \rightarrow L' < \frac{1}{2}$, so $\MDL_{\lambda_m}$ will select $h_1$ as $m \rightarrow \infty$. This implies that $\mathbb{E} L_D(\MDL_{\lambda_m}(S)) \rightarrow L'$ as $m \rightarrow \infty$.  This completes the proof for \autoref{lambda>>m}.
    % Also, since we require $L_S(h_1) \leq
    % \frac{1}{2}$, we have $H(L_S(h_1)) \leq  1$.

    % Plugging in $H(L_S(h_0))$, $H(L_S(h_1))$ into the left hand side of \eqref{eqn39}. Because $\liminf \frac{\lambda_m}{m}> 10$, the left hand side of \eqref{eqn39} is negative for all large $m$, with probability $\geq 1-e^{-2 \sqrt{m}}$.
    
    % The high probability bound implies the bound for expectation:
    % $\mathbb{E} L_D(\MDL_{\lambda_m}(S)) \leq L' + e^{-2\sqrt{m}}$, and $\mathbb{E} L_D(\MDL_{\lambda_m}(S)) \geq L' \cdot (1-e^{-2 \sqrt{m}})$. Thus, $\mathbb{E}_S \left[L_D(\MDL_{\lambda_m}(S))\right] \rightarrow 1$ as sample size $m \rightarrow \infty$.
\end{proof}
\section{Tight Bounds on the Binomial CDF, and the Minimum of i.i.d Binomials, in terms of KL-Divergence}\label{AppendixC}
In both our upper and lower bounds, we rely on a tight bound on the minimum of i.i.d (scaled) Binomials.  These follow standard union bound arguments applied to a tight version of Sanov's Theorem, as presented in an accompanying note \cite{binomials}, and  reproduce here for completeness. 

We first provide a tight upper {\em and} lower bound on the binomial tail:
\begin{lemma}[Binomial tail]
\label{lemma}
    Let $X \sim \frac{1}{n}\textnormal{Bin}(n, p)$ be a scaled Binomial random variable. Then for $a \leq p$,
    \[
    \log \mathbb{P}(X \leq a) \in -n \KL(a\Vert p) \pm \left(4\log(n+1) + \left[\log \frac{p}{1-p}\right]_+\right),
    \]   where $\KL(\alpha\Vert \beta)$ denotes $\KL(Ber(\alpha)\Vert Ber(\beta)) = \alpha\log\frac{\alpha}{\beta} + (1-\alpha)\log\frac{1-\alpha}{1-\beta}$.
\end{lemma}

\begin{proof}
    We write $X = \frac{1}{n}\sum_{i = 1}^n X_i$, where $X_i \overset{\mathrm{iid}}{\sim} \textnormal{Ber}(p)$, and so $X_1, X_2, \cdots, X_n$ is a sequence of $n$ symbols from the alphabet $\mathcal{X} = \{0,1\}$ with type $(1-X,X)$. Denote the true distribution $Q=\textnormal{Ber}(p)$. 
    
    The upper bound follows directly from Sanov's theorem \citep{TC}:
    \begin{equation}\label{eqn1}
        \log \mathbb{P}(X \leq a) \leq -n \KL(a\Vert p) + 2\log(n+1).
    \end{equation}
    To get a finite sample lower bound, we round $a$ to a multiple of $1/n$.  That is, let $k=\lfloor a n \rfloor$ and $\tilde{a}=k/n$, so that $a-1/n < \tilde{a} \leq a$.
   
    Let $\mathcal{P}_n = \{(P(0), P(1)): (\frac{0}{n}, \frac{n}{n}),(\frac{1}{n}, \frac{n-1}{n}), \cdots, (\frac{n}{n}, \frac{0}{n}) \}$ be the set of types with denominator $n$, and $E = \{P: P(1) \leq a\}$. Then the type $P_{\tilde{a}} = (1-\tilde{a}, \tilde{a})$ lies in the intersection $E \cap \mathcal{P}_n$.

   Given the type $P \in \mathcal{P}_n$, let $T(P) = \{x \in \mathcal{X}^n: P_x = P\}$ denote the type class of $P$, which is the set of sequences of length $n$ and and type $P$. Then, by adapting equations (11.104) to (11.106) in the lower bound proof of \cite{TC}, we have:
    
    \begin{equation*}
         \begin{split}
        \mathbb{P}(X \leq a) = Q^n(E) &= \sum_{P \in E \cap \mathcal{P}_n} Q^n\left( T \left( P \right) \right)\\
        &\geq Q^n \left(T\left(P_{\tilde{a}}\right)\right)\\
        &\geq \frac{1}{(n+1)^2} 2^{-n\KL(\tilde{a}\Vert p)}.
    \end{split}
    \end{equation*}

    Taking the logarithm on both sides yields:
    \begin{equation}
     \log \mathbb{P}(X \leq a) \geq -2\log(n+1) - n\KL(\tilde{a}\Vert p). \tag{*} \label{*}
    \end{equation}

        Since $a - \tilde{a} < 1/n$, $H(a) - H(\tilde{a}) < H(\frac{1}{n}) < \frac{2}{n} \log n$. This implies that $\KL(\tilde a\Vert p) - \KL(a\Vert p)= (a-\tilde{a})\log \frac{p}{1-p} + H(a) - H(\tilde{a}) \leq \frac{1}{n} \left[ \log \frac{p}{1-p} \right]_+ + \frac{2}{n} \log n$. Plugging this in the inequality \eqref{*} yields
    \begin{align}
        \log \mathbb{P}(X \leq a) &\geq -2\log(n+1) - n\KL(\tilde{a}\Vert p)\notag\\
        &\geq -2\log(n+1) - \left(n\KL(a\Vert p) + 2\log n + \left[\log \frac{p}{1-p}\right]_+ \right)\notag\\
        &\geq -n \KL(a\Vert p) - 4\log(n+1) - \left[\log \frac{p}{1-p}\right]_+.\label{eq:cdflower}
    \end{align}
    
    The upper bound \eqref{eqn1} and lower bound \eqref{eq:cdflower} together yield the desired result.
\end{proof}
Next, we use the finite sample bound on the Binomial CDF to prove the following concentration bounds of the minimum of i.i.d Binomials in terms of KL divergence.
\begin{theorem}[minimum of i.i.d Binomial]\label{thm12}
    Let $\{X_i\}_{i = 1}^r \overset{\mathrm{iid}}{\sim} \frac{1}{m}\textnormal{Bin}(m, p)$, $Z = \min_{i = 1, \cdots, r} X_i$. Given fixed confidence parameter $\delta \in (0,1)$, let $\Delta(\delta, p, m) = \log \frac{1}{\delta/2} + 4\log (m+1) + \left[ \log \frac{p}{1-p} \right]_+ $. If $\Delta(\delta, p, m) < \log r$, then with probability $1-\delta$, we have
    \[
    Z < p \text{, and } \KL(Z\Vert p) \in \frac{\log r \pm \Delta(\delta, p, m)}{m},
    \]
    except that if $\KL(0 \Vert p) < \frac{\log r -\Delta(\delta, p, m)}{m}$, then with probability $1-\delta$, $Z = 0$.
\end{theorem}

\begin{proof}
    Consider any interval $[a,b]$, such that $a\leq b< p$. Define the following events:
    \begin{eqnarray*}
        U &=& \{ \KL(Z\Vert p) \leq \KL(a\Vert p) \}, \\
        L &=& \{ \KL(Z\Vert p) \geq \KL(b\Vert p) \}, \\
        A &=& \{ Z \geq a \},\textrm{ and} \\
        B &=& \{ Z \leq b \}.
    \end{eqnarray*}
    \begin{center}
       \begin{tikzpicture}
            % Main number line
            \draw[black, thick] (0,0) -- (10,0);
            % Ticks and labels for a, b, and p
            \draw[thick] (2,0.2) -- (2,-0.2) node[below] {$a$};
            \draw[thick] (5,0.2) -- (5,-0.2) node[below] {$b$};
            \draw[thick] (7,0.2) -- (7,-0.2) node[below] {$p$};
            % Magenta line (B)
            \draw[magenta, thick] (0,-1) -- (5,-1) node[midway, above] {$B$};
            % Cyan line (A)
            \draw[cyan, thick] (2,-1.5) -- (10,-1.5) node[midway, above] {$A$};
            % Blue line (U)
            \draw[blue, thick] (2,-2) -- (9,-2) node[midway, above] {$U$};
            % Red lines (L)
            \draw[red, thick] (0,-2.5) -- (5,-2.5);
            \draw[red, thick] (8,-2.5) -- (10,-2.5);
            \node[above] at (4,-2.5) {\textcolor{red}{$L$}};
        \end{tikzpicture}    
    \end{center}
    By the monotonicity of the KL divergence, we have that $B \subseteq L$ and $A \cap B \subseteq U$ (but note that we generally \emph{don't} have $A \subseteq U$). This means that $A \cap B \subseteq U \cap L$, and consequently:
    \[
        \mathbb{P}(U \cap L) \geq  \mathbb{P}(A \cap B)  = 1 - \mathbb{P}(A^\comp) - \mathbb{P}(B^\comp).
    \]
    The theorem will follow from choices of $a$ and $b$ that help bound $\mathbb{P}(A^\comp)$ and $\mathbb{P}(B^\comp)$.
    
    Using the fact that $a<p$, along with the union bound and \autoref{lemma}, we have 
    \[
    \mathbb{P}(A^\comp) =\mathbb{P}(Z < a) \leq \mathbb{P}(Z \leq a) \leq r \cdot\mathbb{P}(X_1 \leq a) \leq r \cdot 2^{-m\KL(a\Vert p) + 4\log (m+1) + \left[ \log \frac{p}{1-p} \right]_+}.
    \]
    Suppose $\KL(0\Vert p) \geq \frac{\log r + \Delta(\delta, p, m)}{m}$. Since $\KL(p\Vert p) = 0$, and KL is continuous by its first argument, by intermediate value theorem, we can choose $0 \leq a <p$ such that
    \begin{equation} \label{eq:KLap} 
    \begin{split}
    \KL(a\Vert p)  &= \frac{\log r + \Delta(\delta, p, m)}{m}\\
    &= \frac{\log r + \log \frac{1}{\delta/2} + 4\log (m+1) + \left[ \log \frac{p}{1-p} \right]_+}{m},
    \end{split}
    \end{equation}
     which gives $r \cdot 2^{-m\KL(a\Vert p) + 4\log (m+1) + \left[ \log \frac{p}{1-p} \right]_+}=\delta/2$.
    Thus, by choosing $0 \leq a < p$ according to \eqref{eq:KLap}, we get $\mathbb{P}(A^\comp) \leq \frac{\delta}{2}$.
    
    If $\KL(0\Vert p) < \frac{\log r + \Delta(\delta, p, m)}{m}$, in other words, there is no $0 \leq a < p$ satisfying \eqref{eq:KLap}, then take $a = 0$. And in this case, the upper bound of the theorem trivially holds for any $Z < p$ because
    \begin{align*}
        \mathbb{P}\left(\KL(b \Vert p) \leq \KL(Z\Vert p) \leq \KL(0 \Vert p) < \frac{\log r + \Delta(\delta, p, m)}{m}\right) &\geq \mathbb{P}(0 \leq Z \leq b)
        = 1 - \mathbb{P}(Z > b).
    \end{align*}

    On the other hand, by the independence of data points, we have:
    \begin{equation}\label{eq:indep}
    \mathbb{P}(B^\comp) = \mathbb{P}(Z > b) = (1-\mathbb{P}(X_1 \leq b))^r.
    \end{equation}
    Using the inequality $\forall x \in [0,1], k > 0: (1-x)^k \leq e^{-kx}$ and \autoref{lemma}, we have
    \begin{equation}\label{eq:expbd}
    (1-\mathbb{P}(X_1 \leq b))^r \leq \exp\left(-r \cdot \mathbb{P}(X_1 \leq b)\right) \leq \exp \left(-r \cdot 2^{-m\KL(b\Vert p) - 4\log (m+1) - \left[ \log \frac{p}{1-p} \right]_+}\right).
    \end{equation}
    Suppose $\KL(0\Vert p) \geq \frac{\log r - \log \ln \frac{1}{\delta/2} - 4\log (m+1) - \left[ \log \frac{p}{1-p} \right]_+}{m}$, again by the intermediate value theorem, we can choose $0 \leq b < p$ such that 
    \begin{equation} \label{eq:KLbp}
    \KL(b\Vert p) = \frac{\log r - \log \ln \frac{1}{\delta/2} - 4\log (m+1) - \left[ \log \frac{p}{1-p} \right]_+}{m}, 
    \end{equation}
    which gives $\exp \left(-r \cdot 2^{-m\KL(b\Vert p) - 4\log (m+1) - \left[ \log \frac{p}{1-p} \right]_+}\right) = \delta/2$.
    Thus, by choosing $0 \leq b <p$ according to \eqref{eq:KLbp}, we get $\mathbb{P}(B^\comp) \leq \frac{\delta}{2}$.

    If $\KL(0\Vert p) < \frac{\log r - \log \ln \frac{1}{\delta/2} - 4\log (m+1) - \left[ \log \frac{p}{1-p} \right]_+}{m}$, in other words, there is no $0\leq b < p$ satisfying \eqref{eq:KLbp}, then by combining \eqref{eq:indep} and \eqref{eq:expbd},
    \[
        \mathbb{P}(Z > 0)\leq \exp \left(-r \cdot 2^{-m\KL(0\Vert p) - 4\log (m+1) - \left[ \log \frac{p}{1-p} \right]_+}\right)
        \leq \frac{\delta}{2}.
    \]
    So in this case, we have with probability $\geq \frac{\delta}{2} > 1 - \delta$, $Z = 0$.

    Therefore, by choosing $a$ and $b$ as above, we get
    \begin{equation*}
        \mathbb{P}\Big(\KL(Z\Vert p) \in \big(\KL(b\Vert p), \KL(a\Vert p)\big)\Big) = \mathbb{P}(U\cap L) \geq 1 - \mathbb{P}(A^\comp) - \mathbb{P}(B^\comp) \geq 1-\delta,    
    \end{equation*}
    with $\KL(a\Vert p)$ and $\KL(b\Vert p)$ as in \eqref{eq:KLap} and \eqref{eq:KLbp} respectively. Except that if $\KL(0\Vert p) < \frac{\log r - \Delta(\delta, p, m)}{m}$, then with probability $> 1 - \delta$, $Z = 0$.
 
    The theorem follows by widening this interval, to get a symmetric expression.
\end{proof}

\end{document}